\documentclass[journal]{IEEEtran}
\ifCLASSINFOpdf
\else
\fi
\usepackage{graphicx}
\usepackage{amsmath}
\usepackage{amsfonts}
\usepackage{amssymb}%
\usepackage{amsthm}
\usepackage{hyperref}
\usepackage{epsfig}
\usepackage{subfigure}
\usepackage{color}
\usepackage{listings}
\usepackage{booktabs}
\usepackage{multirow}
\usepackage{algorithm}

\usepackage{algorithmic}
\usepackage{amsfonts}
\usepackage{multicol}
\usepackage{epstopdf}
\usepackage[justification=centering]{caption}

\usepackage{cleveref}

\newtheorem{theorem}{Theorem}
\newtheorem{lemma}{Lemma}

\newtheorem{proposition}{Proposition}

\newtheorem{assumption}{Assumption}

\begin{document}
%
\title{Fully-Corrective Gradient Boosting with Squared Hinge: Fast Learning Rates and Early Stopping}
%
%
%

\author{Jinshan~Zeng, Min Zhang
        and~Shao-Bo~Lin
\thanks{J. Zeng and M. Zhang are with the School of Computer and Information Engineering, Jiangxi Normal University, Nanchang,  China, 330022. (Email: jsh.zeng@gmail.com (J. Zeng), zhangmin20181010@gmail.com (M. Zhang)) }
\thanks{S.-B. Lin is with the Center of Intelligent Decision-Making and Machine Learning, School of Management, Xi'an Jiaotong University, Xi'an, China, 710049. (Email: sblin1983@gmail.com) }
\thanks{Corresponding author: \textit{Shao-Bo Lin (sblin1983@gmail.com)}}
}
\maketitle

\begin{abstract}
Boosting is a well-known method for improving the accuracy of weak learners in machine learning. However, its theoretical generalization guarantee is missing in literature. In this paper, we propose an efficient boosting method with theoretical generalization guarantees for binary classification. Three key ingredients of the proposed boosting method are: a) the \textit{fully-corrective greedy} (FCG) update in the boosting procedure, b) a differentiable \textit{squared hinge} (also called \textit{truncated quadratic}) function as the loss function, and c) an efficient alternating direction method of multipliers (ADMM) algorithm for the associated FCG optimization. The used squared hinge loss not only inherits the robustness of the well-known hinge loss for classification with outliers, but also brings some benefits for computational implementation and theoretical justification. Under some sparseness assumption, we derive a fast learning rate of the order ${\cal O}((m/\log m)^{-1/4})$ for the proposed boosting method, which can be further improved to ${\cal O}((m/\log m)^{-1/2})$ if certain additional noise assumption is imposed, where $m$ is the size of sample set. Both derived learning rates are the best ones among the existing generalization results of boosting-type methods for classification. Moreover, an efficient early stopping scheme is provided for the proposed method. A series of toy simulations and real data experiments are conducted to verify the developed theories and demonstrate the effectiveness of the proposed method.
\end{abstract}


%
\IEEEpeerreviewmaketitle

\section{Introduction}
%
%
%
%
%
%

Boosting  \cite{Freund1995} is a powerful learning scheme that combines multiple
weak prediction rules to produce a strong learner with the underlying intuition
that one can obtain accurate prediction by combining ``rough'' ones. It has been successfully used in numerous learning tasks  such as
regression, classification, ranking and  recognition \cite{Schapire2012}.
  The
gradient descent view   of boosting \cite{Friedman2000,Friedman2001}, or gradient boosting,  provides a springboard to understand
and improve boosting via connecting   boosting  with  a two-step stage-wise fitting of additive models
corresponding to various loss functions.

There are commonly four ingredients of gradient boosting: a set of weak learners, a loss function, an update scheme and an early stopping strategy.  The weak learner issue focuses on selecting a suitable set of weak learners by regulating the   property of the estimator to be found. Typical examples are decision trees \cite{Hastie2001}, neural networks \cite{Barron2008} and  kernels \cite{Lin2013}.
The loss function issue devotes to choosing an appropriate loss function to enhance the learning performance.
Besides the classical exponential loss in Adaboost \cite{Freund1995},
some other widely used loss functions  are
the logistic loss in Logit-Boosting  \cite{Friedman2000}, least square loss in $L_2$ Boosting \cite{Buhlmann2003} and hinge loss in  HingeBoost \cite{Gao2011}.
The update scheme refers to  how to iteratively
derive   a new estimator based on the selected weak learners. According to the gradient descent view,
there are numerous iterative schemes for boosting \cite{Friedman2001}.  Among these, five
most commonly used iterative schemes are the  original boosting iteration \cite{Freund1995},   regularized
boosting iteration via shrinkage (RSBoosting) \cite{Friedman2001}, regularized boosting via truncation
(RTBoosting) \cite{Zhang2005}, $\varepsilon$-Boosting \cite{Hastie2001}  and re-scaled boosting (RBoosting) \cite{Wang2019}.  Noting that boosting is doomed to over-fit \cite{Buhlmann2003}, the early stopping issue
  depicts how to
terminate the learning process to avoid over-fitting.
Some popular strategies to yield a stopping rule of high quality are \textit{An Information Criterion} (AIC) \cite{Buhlmann2003}, $\ell_0$-based complexity restriction \cite{Barron2008} and $\ell_1$-based adaptive terminate rule.

The learning performance of $L_2$ Boosting has been rigorously verified in regression \cite{Barron2008,Bagirov2010}. In fact, under some sparseness assumption of the regression function, a learning rate of order $\mathcal O(m^{-1/2})$ has been provided for numerous variants of the original boosting \cite{Barron2008,Bagirov2010}, where $m$ denotes the size of data set. However, for classification where $L_2$ Boosting performs practically not so well, there lack tight classification risk estimates for boosting as well as its variants.
For example, the classification risk for  AdaBoost is of an order $\mathcal O((\log m)^{-1})$ \cite{Bartlett2007} and for some variant  of Logit-Boosting    is  of an order  $\mathcal O(m^{-1/8})$  \cite{Zhang2005}.
There are mainly two reasons resulting in such slow learning rates.
The one is that the  original update scheme in boosting leads to slow numerical convergence rate \cite{Livshits2009,Mukherjee2013}, which requires numerous boosting  iterations to achieve a
prescribed
accuracy. The other is that the widely used loss functions such as the exponential loss and logistic loss   do not admit the truncation (or clip) operator like (\ref{Truncation}) below,  requiring   tight uniform bounds for the derived estimator.

The aim of the present paper is to derive tight classification risk bounds for boosting-type algorithms via selecting appropriate iteration scheme and loss function. In fact, we adopt the widely used \textit{fully-corrective greedy} (FCG) update scheme and  the \textit{squared hinge} (also called truncated quadratic) function (i.e., $\phi_{h^2}(t)=\max\{0,1-t\}^2$ for any $t\in \mathbb{R}$) as the loss function.
 FCG update scheme has
  been successfully  used in \cite{Sochman2004,Shalev-Shwartz2010,Johnson2014},
mainly in terms of the fast numerical convergence rate.
Inspired by the square-type inequality \cite{Lin2017}, the squared hinge loss  has been exploited  to ease the computational implementation and is regarded to be an improvement of the classical hinge loss  \cite{Janocha2016,Mangasarian2001,Kanamori2012}.
By taking advantage of the special form of the squared hinge loss, we develop an alternating direction method of multipliers (ADMM) algorithm \cite{Glowinski1975,Gabay1976} for efficiently finding the optimal coefficients of the FCG optimization subproblem.
More importantly, a tight classification risk bound is derived  in the statistical learning framework  \cite{Cucer2007}, provided the algorithm is appropriately early stopped.

In a nutshell, our contributions can be summarized as follows.

$\bullet$ {\it Algorithmic side:} We propose a novel variant of boosting to  improve its performance for binary classification. The
FCG update scheme and squared hinge loss are utilized in the new variant to accelerate the numerical convergence rate and reduce the classification risk.

$\bullet$ {\it Theoretical side:} We derive fast learning rates for the proposed algorithm in binary classification. Under some regular sparseness assumption, the derived learning rate achieves an order of $\mathcal O((m/\log m)^{-1/4})$, which is a new record for boosting classification. If some additional noise condition is imposed, then the learning rate can be further improved to $\mathcal O((m/\log m)^{-1/2})$.

$\bullet$ {\it Numerical side:} We conduct a series of experiments including the toy simulations, UCI-benchmark data experiments and a real-world earthquake intensity classification experiment to show  the feasibility and  effectiveness of the proposed algorithm. Our numerical results show that the proposed
variant of boosting
is at least  comparable  with the state-of-the-art methods.

The rest of this paper is organized as follows.
In Section \ref{sc:algorithm}, we introduce the proposed boosting method in detail.
In Section \ref{sc:generalization-error}, we provide the theoretical generalization guarantees of the proposed method.
A series of toy simulations are conducted in Section \ref{sc:simulation} to illustrate the feasibility of the suggested method,
and some real-data experiments are provided in Section \ref{sc:realdata} to demonstrate the effectiveness of the proposed method.
All the proofs are provided in Section \ref{sc:proofs}.
We conclude this paper in Section \ref{sc:conclusion}.

\section{Proposed method}
\label{sc:algorithm}

In this section, after presenting the classical boosting, we introduce our variant in detail.

\subsection{Boosting}
\label{sc:boosting}

Boosting can be regarded as one of the most important methods in machine learning for classification and regression \cite{Schapire2003}.
The original versions of boosting proposed by \cite{Schapire1990} and \cite{Freund1995} were not adaptive and could not take full advantage of the weak learners.
Latter, an adaptive boosting algorithm called AdaBoost was introduced by \cite{Freund1997} to alleviate many practical difficulties of the earlier versions of boosting. The gradient descent view of boosting \cite{Friedman2000} then connects boosting with the well known greedy-type algorithms  \cite{Temlyakov2008a} equipped with different loss functions.
In light of this perspective, numerous variants of boosting were proposed to improve its learning performance \cite{Hastie2001}.

Given a data set  $D=\{(x_i,y_i)\}_{i=1}^m$ with size $m$, boosting starts with a set of weak learners  ${\cal G}_n := \{g_j\}_{j=1}^n$ with size $n$ and a loss function $\phi$. Mathematically, it formulates the learning problem to find a  function $f\in \mathrm{span}{\cal G}_n$ to minimize the following empirical risk
\begin{align}
\label{Eq:Empirical-risk}
{\cal E}_{D}^\phi(f):= \frac{1}{m} \sum_{i=1}^m \phi(y_i,f(x_i)),
\end{align}
where $\mathrm{span}{\cal G}_n$ represents the function space spanned linearly by ${\cal G}_n$.
If $\phi$ is  Fr\'{e}chet
differentiable, gradient boosting   firstly finds a
$g_k^*\in {\cal G}_n$ such that
\begin{equation}\label{Gradient-proj}
              - (\nabla {\cal E}^\phi_D(f_{k-1}),g_k^*)
              =\sup_{g\in {\mathcal G}_n} -(\nabla {\cal E}^\phi_D(f_{k-1}),g),
\end{equation}
  where $(\nabla {\cal E}^\phi_D(f),h)$ denotes  the
value of linear functional $\nabla {\cal E}^\phi_D(f)$ at $h$.
Then,  it finds a $\beta_k^*\in \mathbb R$  such that
\begin{equation}\label{Linear search}
       {\cal E}^\phi_D(f_{k-1}+\beta_k^*g^*_k)= \inf_{\beta_k\in\mathbb
       R}{\cal E}^\phi_D(f_{k-1}+\beta_kg^*_k).
\end{equation}
In this way, gradient boosting yields a set of estimators $\{f_k\}_{k=1}^\infty$ iteratively and early stops the algorithm according to the bias-variance trade-off \cite{Zhang2005} to get the final estimator $f_{k^*}$, where $k^*$ is the terminal number of iterations.

According to the above description, it can be noted that the selections of weak learners, the loss function, update scheme and early stopping strategy  play important roles in the practical implementation of gradient boosting.
The studies in \cite{Bagirov2010,Johnson2014,Friedman2001,Buhlmann2003,Barron2008,Wang2019,Zhang2005,Lin2013} discussed  the importance of the mentioned four issues respectively and then presented numerous variants of boosting accordingly.

\subsection{Fully-corrective greedy update scheme}
\label{sc:boosting-FCG}
There are roughly two approaches to improve the learning performance of boosting:  variance-based method and bias-based method.  The former focuses on controlling  the structure ($\ell^1 $ norm) of the derived boosting estimator and then reduces the variance of boosting for a fixed number of iterations, while the later  devotes to accelerating  the numerical convergence rates and early stopping the iteration procedure.
Among these existing variants of boosting,
RSBoosting  \cite{Friedman2001},
RTBoosting  \cite{Zhang2005} and $\varepsilon$-Boosting \cite{Hastie2001} are typical variance-based methods, while RBoosting \cite{Wang2019} is a bias-based method.
 The problem is, however, that the variance-based method frequently requires a large number of iterations to achieve a desired accuracy,
while the bias-based method may suffer from the same problem unless some additional boundedness assumptions are imposed.
An intuitive experiment is provided in Fig. \ref{fig:comp-boost}.
From Fig. \ref{fig:comp-boost}(b), it follows that numerous iterations are required for these existing boosting methods to select a small number of weak learners.
\begin{figure}[!t]
\begin{minipage}[b]{0.49\linewidth}
\centering
\includegraphics*[scale=0.32]{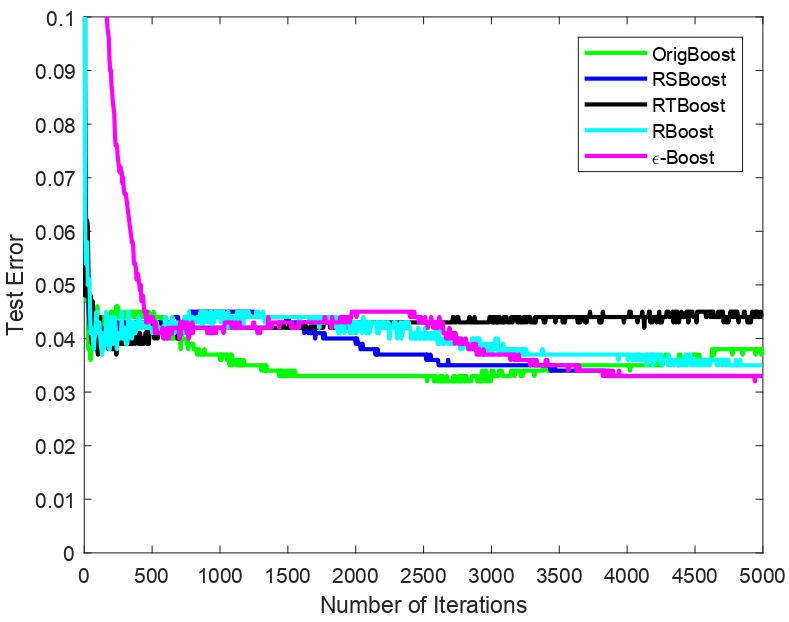}
\centerline{{\small (a) Test error}}
\end{minipage}
\hfill
\begin{minipage}[b]{0.49\linewidth}
\centering
\includegraphics*[scale=0.32]{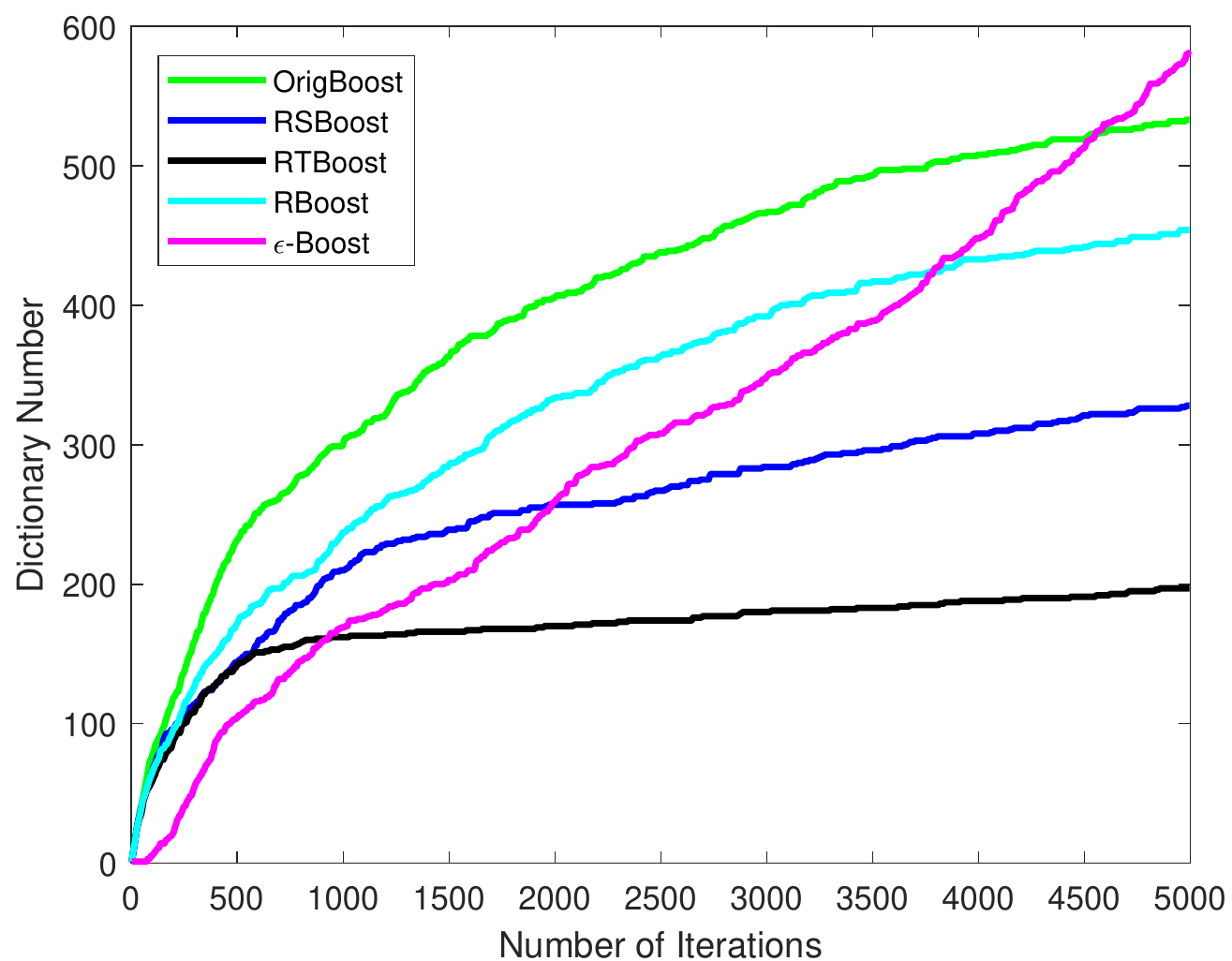}
\centerline{{\small (b) Selected weak learners}}
\end{minipage}
\hfill \caption{ Comparisons on boosting schemes. (a) The curves of test error for different types of boosting methods, (b) the curves of the number of selected weak learners. The detailed experimental settings can be found in Simulation IV in Sec.\ref{sc:experiment-set}.
} \label{fig:comp-boost}
\end{figure}

Different from above variants that can repeatedly select the same weak learners during the iterative procedure, the
 fully corrective greedy (FCG) update scheme proposed in \cite{Johnson2014} finds an optimal combination of all the selected weak learners. In particular, let
$\{g_j^*\}_{j=1}^k$ be the selected weak learners at the current iteration, fully corrective greedy boosting (FCGBoosting) builds an estimator
  via the following minimization
\begin{align}
\label{Eq:FCG}
    f_{D,k}= {\arg \min}_{f\in\mathrm{span}\{g_j^*\}_{j=1}^k} {\cal E}^\phi_{D}(f).
\end{align}
It should be pointed out that FCGBoosting is  similar to the orthogonal greedy algorithm in approximation theory \cite{Temlyakov2008a},
 fully-corrective variant of Frank-Wolfe method   in optimization \cite[Algorithm 4]{Jaggi2013},
and also   orthogonal matching pursuit in signal processing \cite{Tropp2007}.
The advantages of FCGBoosting lie in the sparseness of the derived estimator and the fast numerical convergence rates without any compactness assumption  \cite{Shalev-Shwartz2010}.

\subsection{Squared hinge loss}
\label{sc:loss}

Since the gradient descent viewpoint connects
the gradient boosting  with   various loss functions,
numerous loss functions have been employed in boosting to enhance the performance.
Among these, the exponential loss in AdaBoost, logistic loss in LogitBoost  and square loss in $L_2$ Boosting are the most popular ones.
In the classification settings,
the consistency of AdaBoost and LogitBoost has been proved in \cite{Zhang2005,Bartlett2007}  with  relatively slow learning rates.
In this paper, we equip boosting with the squared hinge  to improve the  learning performance, both in theory and experiments.

As shown in \cite{Lin2017},   the squared hinge   is of quadratic type, and thus theoretically behaves similar to the square loss and commonly better than the other typical loss functions including the exponential loss, logistic loss and hinge loss.
Furthermore, learning with the squared hinge loss usually permits the margin principle \cite{Kanamori2012} and thus practically performs better than the square loss for classification.
Selecting the loss function $\phi$ as the squared hinge loss, i.e., $\phi_{h^2}(t) = (\max\{0,1-t\})^2$ in FCGBoosting, we can obtain a new variant of boosting summarized in Algorithm \ref{alg:boosting}.

\begin{algorithm}\caption{FCGBoosting with squared hinge loss}\label{alg:boosting}
\begin{algorithmic}
\STATE{ {\bf Input}: training sample set $D:= \{x_i,y_i\}_{i=1}^m$, and a dictionary set ${\cal G}_n := \{g_j\}_{j=1}^n$.}
\STATE{Initialization:$f_{D,0}=0, {\cal T}^0 = \emptyset$.}
\STATE{ for $k=1,2,\ldots,$}
\STATE{\ \ \ let $ g_{j_k} = \arg\max_{g_j \in {\cal G}_n} \ -(\nabla {\cal E}^{\phi_{h^2}}_D(f_{k-1}),g_{j})$,}
\STATE{\ \ \ let ${\cal T}^k = {\cal T}^{k-1} \cup \{j_k\}$, and
      \begin{equation}\label{FCB}
\text{(FCG)} \ \      f_{D,k} = {\arg\min}_{f\in \mathrm{span} \{g_j\}_{j\in {\cal T}^k}} {\cal E}^{\phi_{h^2}}_{D}(f).
      \end{equation}
End until the stopping criterion is satisfied.}
\end{algorithmic}
\end{algorithm}

Notice that the \textit{FCG} step \eqref{FCB} in Algorithm \ref{alg:boosting} is a smooth convex optimization problem, a natural algorithmic candidate is the gradient descent (GD) method.
However, as shown in Fig. \ref{fig:ADMM-vs-GD}, GD   needs many iterations to guarantee the convergence, which might be not efficient for the proposed boosting method, since the problem \eqref{FCB} in the FCG step should be solved at each iteration and there are usually numerous iterations for the proposed boosting method.
Instead, we use the alternating direction method of multipliers (ADMM) due to its high efficiency and fast convergence in practice \cite{Gabay1976,Glowinski1975,He2012} (also, shown by Fig. \ref{fig:ADMM-vs-GD}).
The convergence of the suggested ADMM algorithm (presented in Algorithm \ref{alg:ADMM} in Appendix A) and its $\mathcal O(1/t)$ rate of convergence have been established in the existing literature (say, \cite{Gabay1976,Glowinski1975,He2012}).

\begin{figure}[!t]
\begin{minipage}[b]{0.98\linewidth}
\centering
\includegraphics*[scale=.5]{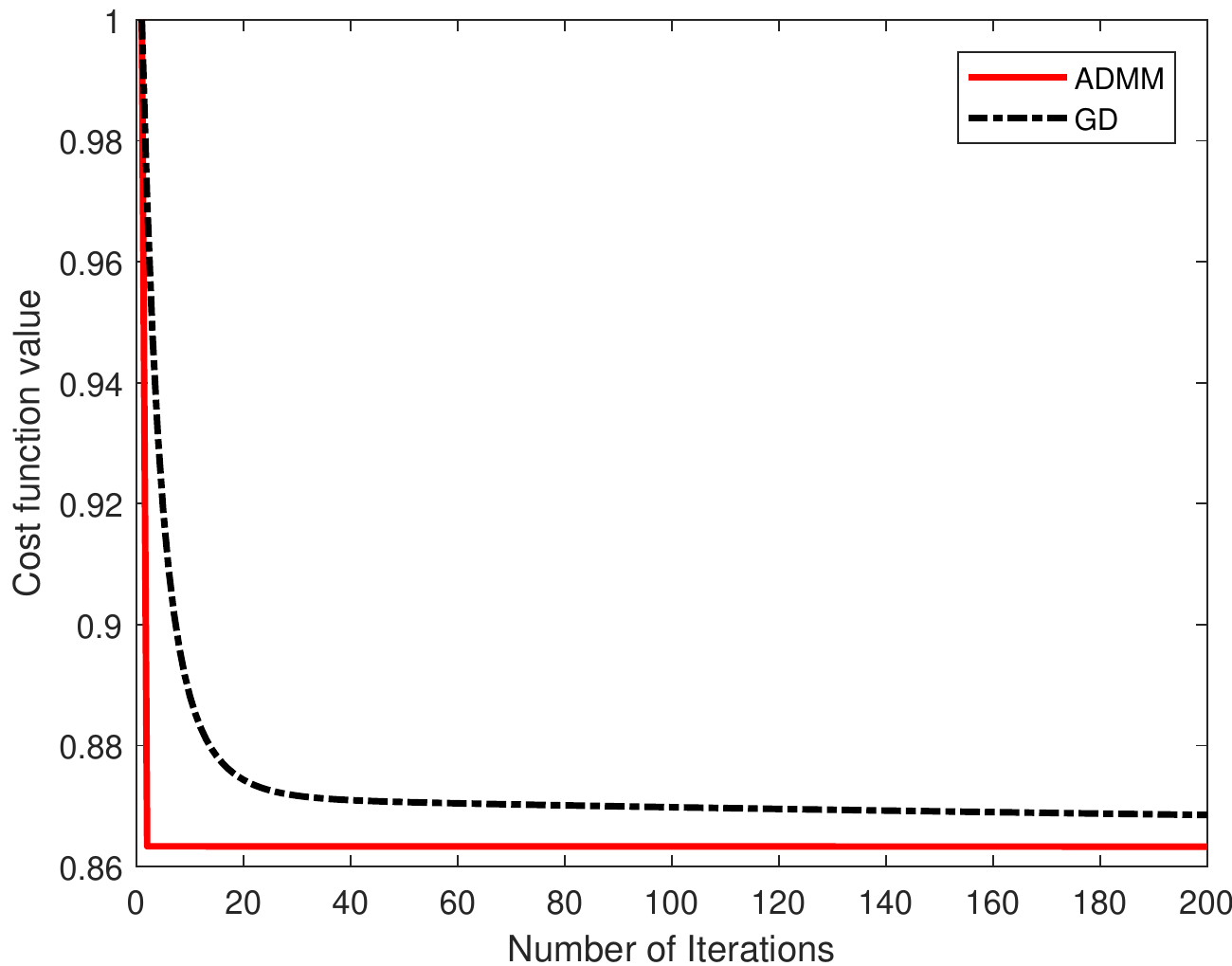}
\end{minipage}
\hfill
\caption{ Comparison on the efficiency of ADMM and  GD for problem \eqref{FCB}. The samples were generated according to Section \ref{sc:experiment-set} with 30\% uniform random noise with $m=1000.$ The matrix $A$ was formed by the Gaussian kernel dictionary with width   0.1 and   dictionary size  15. The computational time of ADMM is 0.034 seconds, while that of GD is 0.53 seconds. It can be observed that ADMM generally converges faster and more efficiently with a lower cost function value in the concerned optimization problem \eqref{FCB}.}
\label{fig:ADMM-vs-GD}
\end{figure}

\section{Generalization Error Analysis}
\label{sc:generalization-error}

In learning theory \cite{Cucer2007,Steinwart2008}, the sample set
$D=\{(x_i,y_i)\}_{i=1}^{m}$ with $x_i\in X$ and $y_i\in
Y=\{-1,1\}$ are drawn independently according to  an unknown
  distribution $\rho$ on $Z:=X\times Y$. Binary
classification algorithms produce a classifier $\mathcal
C:X\rightarrow Y$, whose generalization ability is measured by  the
misclassification error
$$
      \mathcal R(\mathcal C)=\mathbb P [\mathcal C(x)\neq y]=\int_X
      \mathbb P[y\neq\mathcal C(x)|x]d\rho_X,
$$
 where $\rho_X$ is the marginal distribution of $\rho$
and $\mathbb P[y|x]$ is the conditional probability at $x\in X$. The
Bayes rule
 $
         f_c(x) =\mathrm{sgn}(\eta(x)-1/2)$
  minimizes the misclassification error, where
$\eta(x)=\mathbb P[y=1|x]$ is the Bayes decision function and  $\mathrm{sgn}(t)=1$ if $t\geq 0$ and otherwise, $\mathrm{sgn}(t)=-1$.
Since
$f_c$ is independent of the classifier $\mathcal C$, the performance
of $\mathcal C$ can be measured by the excess misclassification
error $\mathcal R(\mathcal C)-\mathcal R(f_c)$. For the derived estimator $f_{D,k}$ in Algorithm \ref{alg:boosting}, we have $\mathcal C=\mathrm{sgn}(f_{D,k}(x))$.
Then, it is sufficient to present a bound for $f_{D,k}(x)-(\eta(x)-1/2)$. With this, we at first present a sparseness assumption on $\eta(x)-1/2$.

%
\begin{assumption}
\label{assump:approx-error}
There exists an $h_0 \in \mathrm{span}{\cal G}_n$ such that
\begin{align}
\label{Eq:approx-error}
\|\eta-1/2 - h_0\|_{\rho} \leq C_1 n^{-r}, \ \text{and} \ \|h_0\|_{\ell^1} \leq C_2,
\end{align}
for some positive constants $r, C_1, C_2$.
\end{assumption}

Assumption \ref{assump:approx-error} requires that   $\eta(x)-1/2$ should be sparsely approximated by the set of weak learners with certain fast decay of some polynomial order.
Such an assumption is regular in the analysis of boosting algorithm and has been adopted in large literature \cite{Zhang2005,Bartlett2007,Barron2008,Temlyakov2008a,Livshits2009,Bagirov2010,Shalev-Shwartz2010,Lin2013,Mukherjee2013,Wang2019}.
Under this assumption, we can derive the following learning rate for FCGBoosting
\begin{theorem}
\label{Thm:main-corollary}
Let ${\cal G}_n := \{g_j\}_{j=1}^n$ be a set of weak learners with $\|g_j\|_\infty\leq 1$, $j=1,\dots,n$.  Under Assumption \ref{assump:approx-error}, if $n\sim m^a$ for $a\geq 1$, $r \geq \frac{1}{4a}$ and $k\sim\sqrt{\frac{m}{\log m}}$, then
  for any $0<\delta<1$, with confidence at least $1-\delta$, there holds
\begin{align*}
{\cal R}(\mathrm{sgn}( f_{D,k})) - {\cal R}(f_c)\leq  C_3 \left(\frac{m}{\log m} \right)^{-1/4} \log \frac{4}{\delta},
\end{align*}
where $C_3$ is a positive constant independent of $\delta$ or $m$.
\end{theorem}

The proof of this theorem will be  presented in Section VI.
This theorem provides some early stopping of the proposed version of boosting method under the assumption that the Bayes decision function can be well approximated by combining weak learners.
From Theorem \ref{Thm:main-corollary}, an optimal $k$ should be set as in the order of ${\cal O}(\sqrt{\frac{m}{\log m}})$,
which shows that the number of selected weak learners  is significantly less than $m$ and $n$.  It should be noted that the derived learning rate in Theorem \ref{Thm:main-corollary} is of the same order of FCGBoosting with the square loss \cite{Barron2008} under the same setting. To further improve the learning rate, the following Tsybakov noise condition \cite{Tsybakov2004} is generally required.

\begin{assumption}
\label{assump:Tsybakov-noise}
Let $0\leq q \leq \infty.$ There exists a positive constant $\hat{c}_q$ such that \begin{align*}
\rho_X(\{x\in X: |\eta(x)-1/2| \leq \hat{c}_q t\}) \leq t^q, \quad \forall t>0.
\end{align*}
\end{assumption}

The Tsybakov noise assumption measures
  the size of the set of points that
 are corrupted with
high noise in the labeling process, and
  always holds for $q=0$ with
$\hat{c}_q=1$. It is a standard noise assumption in classification which has been adopted in \cite{Steinwart2007,Xiang2011,Lin2017,Zeng2019} to derive  fast learning rates for classification algorithms.
Under the Tsybakov noise assumption, we can improve the rate as follows.
\begin{theorem}
\label{Thm:main-Tsybakov}
Let ${\cal G}_n := \{g_j\}_{j=1}^n$ be a set of weak learners with $\|g_j\|_\infty\leq 1$, $j=1,\dots,n$.  Under Assumption \ref{assump:approx-error} and  Assumption \ref{assump:Tsybakov-noise}  with   $0\leq q<\infty$, if $n\sim m^a$ for $a\geq 1$, $r \geq \frac{1}{4a}$ and
$k\sim\sqrt{\frac{m}{\log m}}$.
Then for any $0<\delta<1$, with confidence at least $1-\delta$, there holds
\begin{align*}
{\cal R}(\mathrm{sgn}(  f_{D,k})) - {\cal R}(f_c)\leq  C_4 \left(\frac{m}{\log m} \right)^{-\frac{q+1}{2(q+2)}} \log \frac{4}{\delta},
\end{align*}
where $C_4$ is a positive constant independent of $\delta$ or $m$.
\end{theorem}
The proof of this theorem is also postponed to Section VI.
Note that when $q=0$, the learning rate established in Theorem \ref{Thm:main-Tsybakov} reduces to that of Theorem \ref{Thm:main-corollary},
while when $q=\infty$, the obtained learning rate in Theorem \ref{Thm:main-Tsybakov} approaches to $\left(\frac{m}{\log m} \right)^{-1/2}$,
which is a new record for the boosting-type methods under the classification setting.

\section{Toy Simulations}
\label{sc:simulation}

In this section, we present a series of toy simulations to demonstrate   the feasibility and effectiveness of FCGBoosting.
All the numerical experiments were carried out in Matlab R2015b environment
running Windows 8, Intel(R) Xeon(R) CPU E5-2667 v3 @ 3.2GHz 3.2GH.

\subsection{Experimental settings}
\label{sc:experiment-set}

The settings of simulations are similar to that in \cite{Zeng2019} described as follows.



{\bf Samples}: In simulations, the training  samples were generated as follows.
Let
$$
   \zeta(t)=\left((1-2t)_+^5(32t^2+10t+1)+1\right)/2,\qquad t\in[0,1]
$$
be a nonlinear Bayes rule.
Let ${\bf x}=\{x_i\}_{i=1}^m \subset ([0,1]\times [0,1])^m$
be drawn i.i.d. according to the uniform
distribution with size $m$.
Then we labeled the samples lying in the epigraph of function $\zeta(t)$ as the positive class, while the rest were labeled as the negative class,
that is, given an $x_i = (x_i(1),x_i(2))$, its label $y_i =1$ if $x_i(2) \geq \zeta(x_i(1))$, and otherwise, $y_i=-1.$
Besides the uniformly random noise generally considered in regression, we mainly focused on the outlier noise in our simulations, that is,
the noisy samples lying in the region that is far from the Bayes (see, Fig. \ref{fig:data}).
We considered different widths (i.e., \textit{tol}) and noise ratios (i.e., \textit{ratio}) with the banded region.


\begin{figure}[!t]
\begin{minipage}[b]{0.49\linewidth}
\centering
\includegraphics*[scale=0.34]{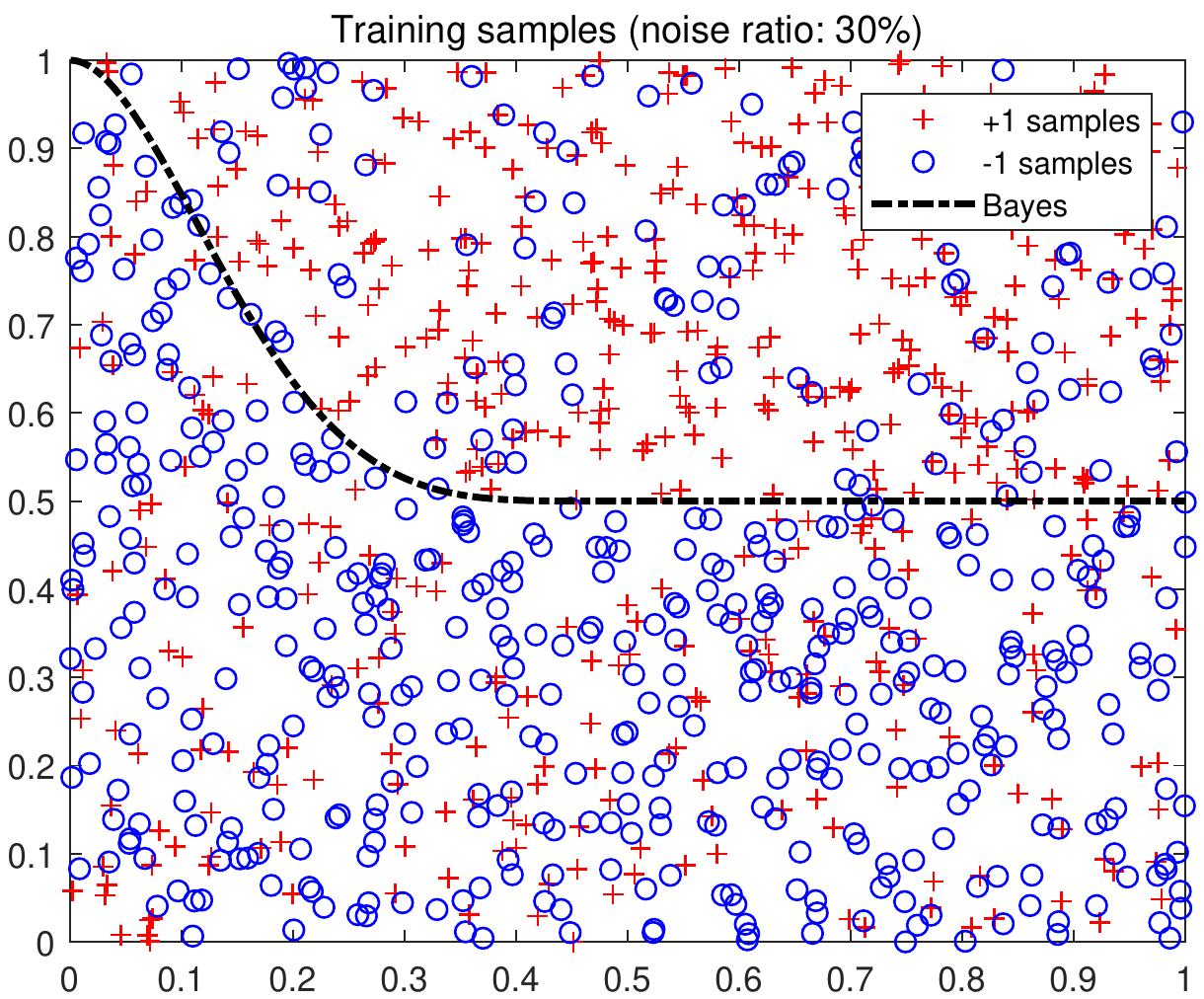}
\centerline{{\small (a) uniform random noise}}
\end{minipage}
\hfill
\begin{minipage}[b]{0.49\linewidth}
\centering
\includegraphics*[scale=0.34]{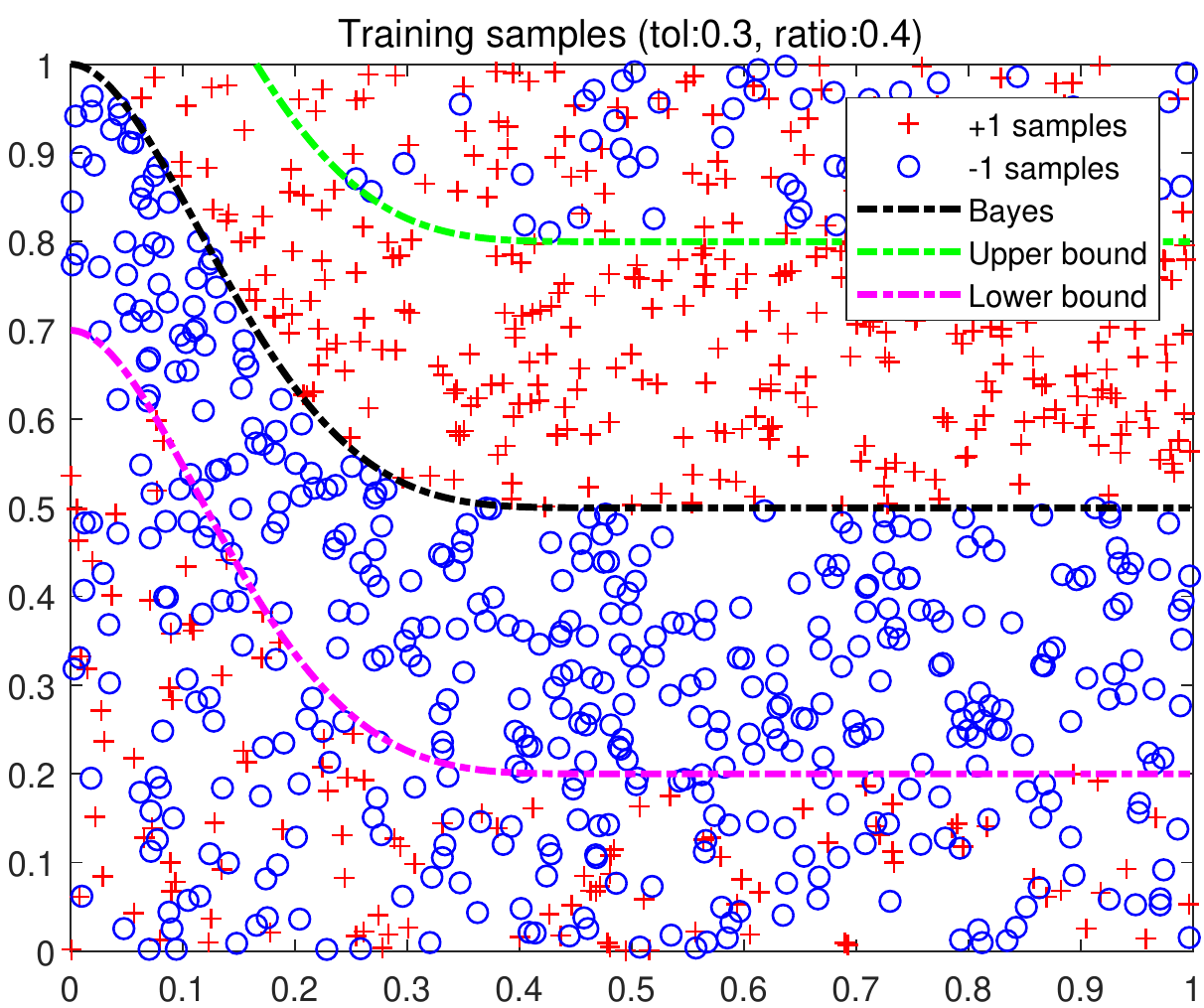}
\centerline{{\small (b) outlier noise}}
\end{minipage}
\hfill \caption{ The generated training samples with \textit{uniform random} noise (see, figure (a)) and \textit{outlier} noise (see, figure (b)) used in simulations. The red points are labeled as ``+1'' class, while the blue points are labeled as ``-1'' class. In the title of figure (b), the notation \textit{tol=0.3} represents that the difference between the Bayes rule and upper (lower) bound is 0.3, while the notation \textit{ratio=0.4} represents that the noise ratios in both left-lower and right-upper regions are 0.4. The total noise level in this case is 17.4\%.
} \label{fig:data}
\end{figure}

{\bf Implementation and Evaluation}:
We implemented four simulations to illustrate the effect of parameters and show the effectiveness of the proposed version of boosting method.
For each simulation, we repeated $20$ times of experiments and recorded the test error, which is defined as the ratio of the number of misclassified test labels to the test sample size.
The first one is to illustrate the effect of the number of iterations $k$, which is generally exploited for setting the stopping rule of the proposed method.
The second one is to demonstrate the effect of the number of dictionaries $n$.
The third one is to show the feasibility and effectiveness of the used squared hinge loss (i.e., $\phi_{h^2}(x) = (1-x)_+^2$) via comparing with some other loss functions including the square loss with $\phi_{sq}(x)=(1-x)^2$, the hinge loss with $\phi_{h}(x)=(1-x)_+$ and the cubed hinge loss with $\phi_{h^3}(x)=(1-x)_+^3$ \cite{Janocha2016}.
The final one is to show the outperformance of the fully-corrective update scheme via comparing to the existing popular update scheme used in gradient boosting.
Since the performance of boosting-type methods also depends on the dictionary type, in this paper, we considered four types of dictionaries, that is, the dictionaries formed by the Gaussian kernel, polynomial kernel, and the neural network kernels with sigmoid and rectified linear unit (ReLU) activations, respectively, and henceforth, they are respectively called \textit{Gauss, polynomial, sigmoid, Relu} for short.
We set empirically the parameters of ADMM algorithm (i.e., Algorithm \ref{alg:ADMM}) used in the FCG optimization step as follows: $\alpha =1$, $\gamma=1$ and the maximal number of iterations was set as 100.

\subsection{Simulation Results}
\label{sc:simul-result}

In this part, we report the experimental results and present some discussions.

{\bf Simulation 1: On effect of number of iterations $k$.}
From Algorithm \ref{alg:boosting}, the number of iterations $k$ is a very important algorithmic parameter, which is generally set as the stopping rule of the boosting type of methods.
By Theorems \ref{Thm:main-corollary} and  \ref{Thm:main-Tsybakov}, a moderately large $k$ (i.e., $k \sim \sqrt{\frac{m}{\log m}}$) is required to achieve the optimal generalization performance.
To illustrate the effect of the number of iterations $k$, we randomly generated training and test samples with both sizes being $m=1000$.
We considered both noise types in training samples, i.e., uniformly random noise with the noise level 30\%, and the outlier noise with $tol=0.3$ and $ratio = 0.4$ (in this case, the level of outlier noise is 17.4\%), as described in Section \ref{sc:experiment-set} .
Moreover, we considered four different dictionaries formed by Gaussian kernel, polynomial kernel, neural network kernel with sigmoid and neural network kernel with ReLU activation, respectively, where the sizes of all four dictionaries are the same $n=1000$.
We varied $k$ according to the set of size 11, i.e.,
$\left\{\left\lceil \frac{1}{2}\sqrt{\frac{m}{\log m}}\right \rceil, \left\lceil \sqrt{\frac{m}{\log m}} \right\rceil, 2 \left\lceil \sqrt{\frac{m}{\log m}} \right \rceil,
\ldots, 10 \left\lceil \sqrt{\frac{m}{\log m}} \right \rceil\right\}$,
and recorded the associated test error.
In this experiment, $\left\lceil \sqrt{\frac{m}{\log m}} \right\rceil = 13$ since $m=1000$.
The curves of test error are shown in Fig. \ref{fig:effectofk}.

From Fig. \ref{fig:effectofk}, the trends of test error for different dictionaries are generally similar, that is, as $k$ increasing from $\left\lceil \frac{1}{2}\sqrt{\frac{m}{\log m}}\right \rceil$ to $10 \left\lceil \sqrt{\frac{m}{\log m}} \right \rceil$, the test error generally decreases firstly and then becomes stable.
This phenomenon is mainly due to that when $k$ is small, the selected model might be under-fitting, and then increasing $k$ shall improve the generalization ability.
More specifically, in both uniform and outlier noise cases, it is generally sufficient to set the iteration number $k$ as $5 \left\lceil \sqrt{\frac{m}{\log m}} \right \rceil$ by Fig. \ref{fig:effectofk}.
This in some extent verifies our main theorems (i.e., Theorems \ref{Thm:main-corollary} and \ref{Thm:main-Tsybakov}), which show that the moderately large $k$ is in the order of $\left\lceil \sqrt{\frac{m}{\log m}} \right \rceil$.
Motivated by this experiment, in practice, the maximal number of iterations $k$ for the proposed boosting method can be empirically chosen from these five values $\left\{\left\lceil \sqrt{\frac{m}{\log m}} \right \rceil, 2 \left\lceil \sqrt{\frac{m}{\log m}} \right \rceil,\ldots, 5\left\lceil \sqrt{\frac{m}{\log m}} \right \rceil \right\}$ via cross validation.
When comparing with these differen dictionaries, the generalization performance of \textit{Relu} are slightly better than the other three dictionaries in both noise cases.


\begin{figure}[!t]
\begin{minipage}[b]{0.49\linewidth}
\centering
\includegraphics*[scale=.32]{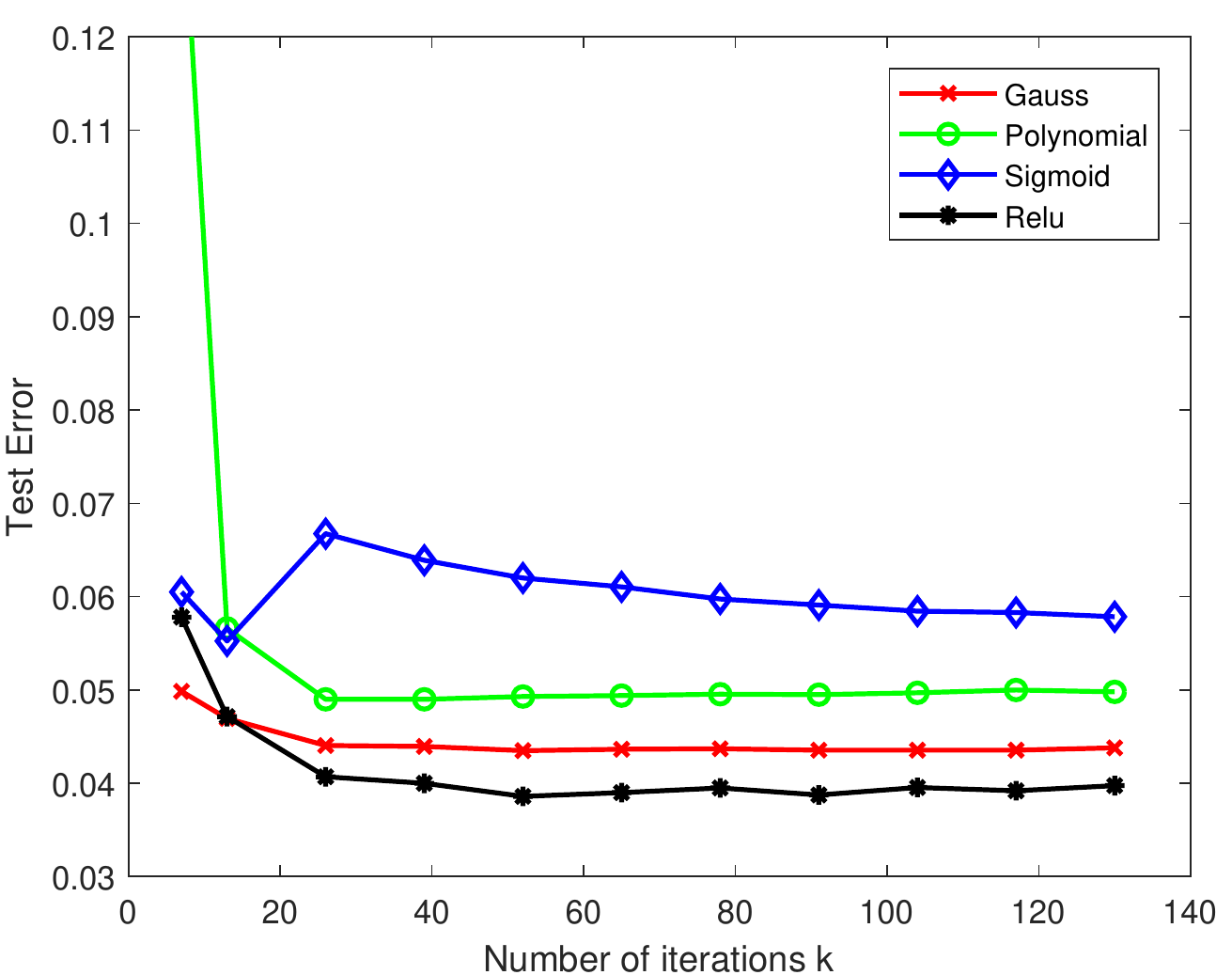}
\centerline{{\small (a) 30\% Uniform random noise}}
\end{minipage}
\hfill
\begin{minipage}[b]{0.49\linewidth}
\centering
\includegraphics*[scale=0.32]{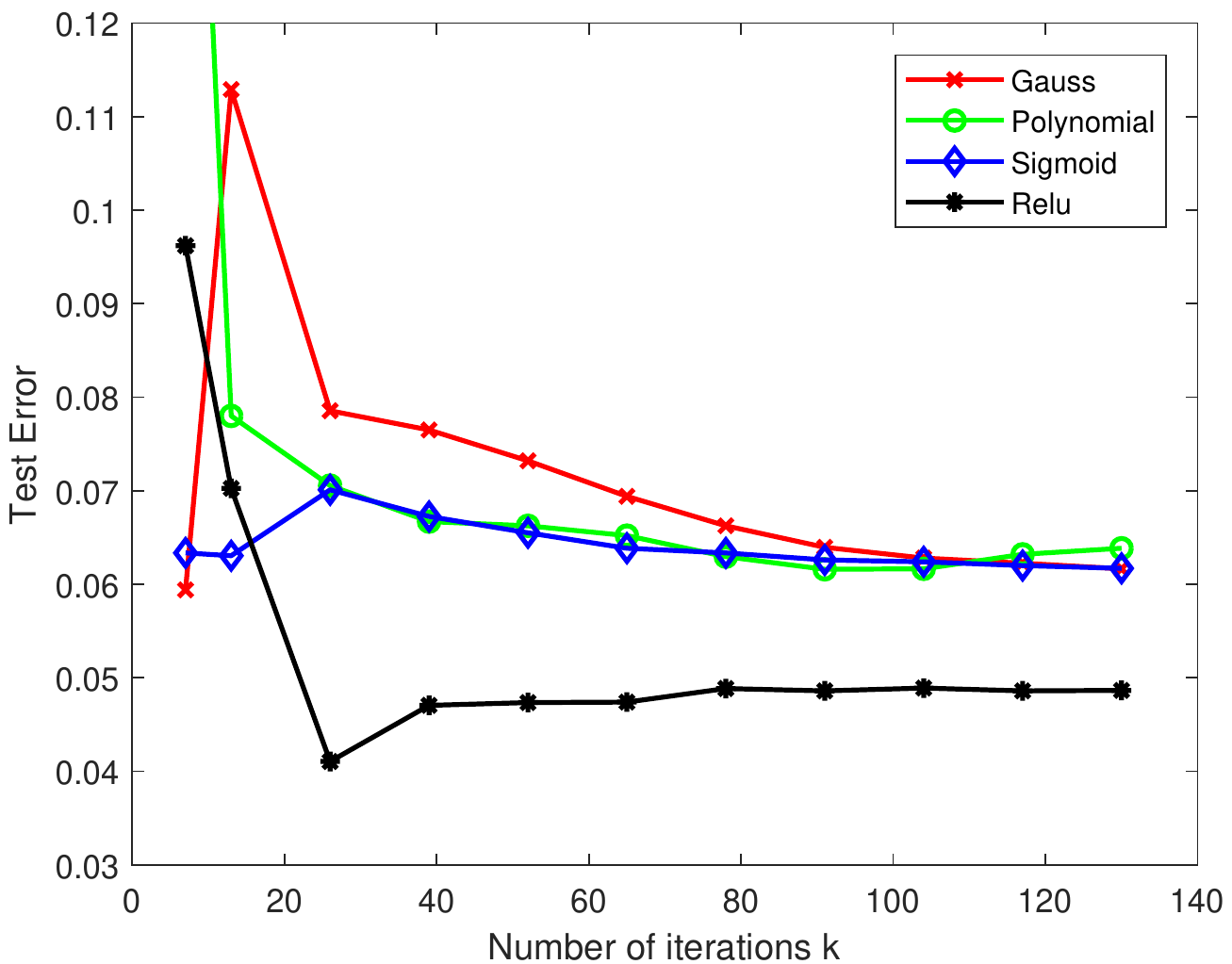}
\centerline{{\small (b) 17.4\% Outlier noise}}
\end{minipage}
\hfill
\caption{ Effect of the number of iterations $k$. (a) the curves of test error of four different dictionaries with respect to $k$ under the 30\% uniform random noise, (b) the curves of test error under the 17.4\% outlier noise where $tol=0.3$ and $ratio = 0.4$.
}
\label{fig:effectofk}
\end{figure}

{\bf Simulation 2: On effect of size of dictionary set $n$.}
Given a dictionary type, the size of dictionary set $n$ generally reflects the approximation ability of the given dictionary set.
Particularly, according to Assumption \ref{assump:approx-error}, one prerequisite condition for the boosting type methods is that the underlying learner should be well-approximated by the chosen dictionary set.
However, a larger dictionary set usually brings more computational cost.
Thus, it is meaningful to verify the possible optimal size of dictionary set.
To illustrate this, in this experiment, the training and test samples were generated in the same way of \textbf{Simulation 1}.
Instead of varying the number of iterations $k$, we varied the size of dictionary set $n$ from $m$ to $10m$, where $m=1000$ is the size of training samples. For each $n$, the number of iterations $k$ was chosen from these five values $\left\{\left\lceil \sqrt{\frac{m}{\log m}} \right \rceil, 2 \left\lceil \sqrt{\frac{m}{\log m}} \right \rceil,\ldots, 5\left\lceil \sqrt{\frac{m}{\log m}} \right \rceil \right\}$ via cross validation.
The curves of test error with respective to the number of dictionary sets $n$ are shown in Fig. \ref{fig:effectofn}.

From Fig. \ref{fig:effectofn}, the number of dictionary set $n$ has little effect on the generalization performance for all types of dictionaries and in both noise settings, when $n$ is in the order of ${\cal O}(m)$. This is also verified by our main theorems (i.e., Theorems \ref{Thm:main-corollary} and \ref{Thm:main-Tsybakov}). By Theorems \ref{Thm:main-corollary} and \ref{Thm:main-Tsybakov}, $n$ should be in the order of  ${\cal O} (m^a)$ for some $a\geq 1$ to achieve the optimal learning rates.
Specifically, in this experiment, we show that the generalization performance of the proposed method does not vary very much when $n$ varies from $m$ to $m^{4/3}$.
In our latter experiments, we empirically set $n=m$ in the consideration of both generalization performance and computational cost.
Regarding the performance of the different dictionaries, we observed that the performance of \textit{Relu} and \textit{Gauss} is slightly better than that of \textit{polynomial} and \textit{sigmoid}.

\begin{figure}[!t]
\begin{minipage}[b]{0.49\linewidth}
\centering
\includegraphics*[scale=.32]{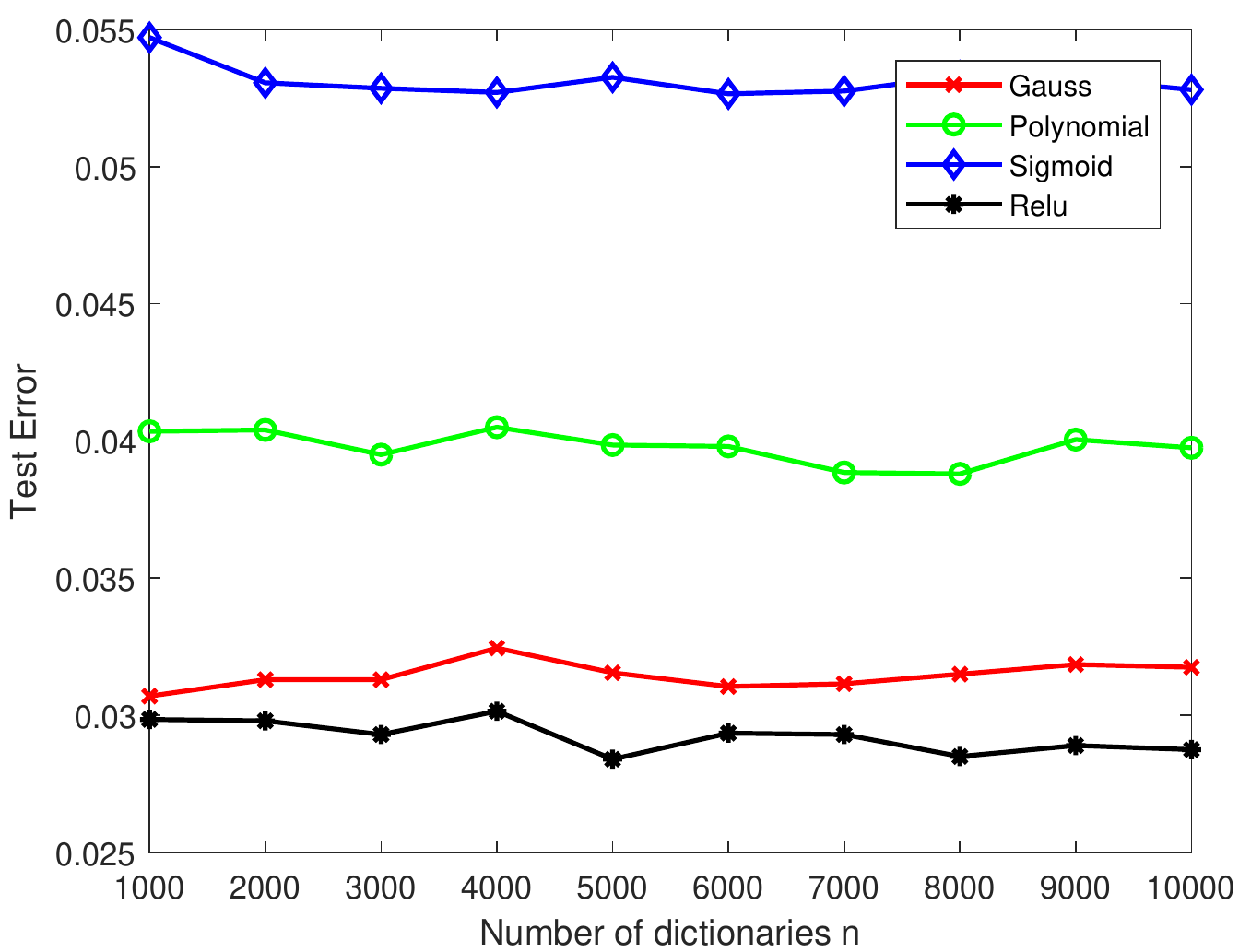}
\centerline{{\small (a) 30\% Uniform random noise}}
\end{minipage}
\hfill
\begin{minipage}[b]{0.49\linewidth}
\centering
\includegraphics*[scale=0.32]{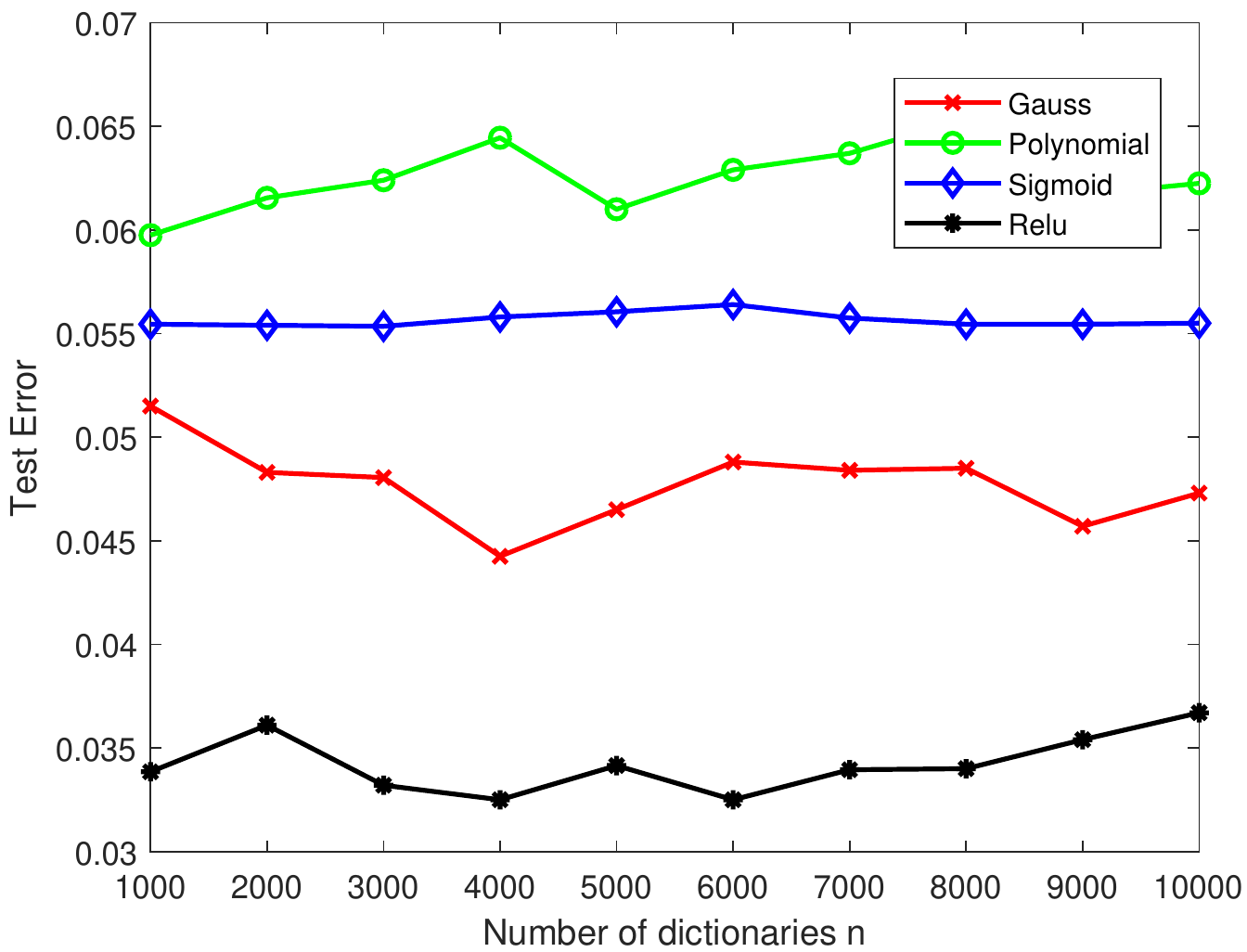}
\centerline{{\small (b) 17.4\% Outlier noise}}
\end{minipage}
\hfill
\caption{ Effect of the number of dictionaries $n$. (a) the curves of test error of four different dictionaries with respect to $n$ under the 30\% uniform random noise, (b) the curves of test error under the 17.4\% outlier noise where $tol=0.3$ and $ratio = 0.4$.
}
\label{fig:effectofn}
\end{figure}

{\bf Simulation 3: On comparison among different losses.}
In this experiment, we compare the performance of the fully-corrective greedy boosting method with different loss functions, including the square loss $\phi_{sq}(x)= (1-x)^2$, the hinge loss $\phi_h(x)=(1-x)_+$, the suggested \textit{squared hinge} loss $\phi_{h^2}(x) = (1-x)_+^2$ and the cubed hinge loss $\phi_{h^3}=(1-x)_+^3$. Note that the hinge loss is non-differentiable, while both the squared hinge and cubed hinge losses are differentiable. As demonstrated in the literature \cite{Mangasarian2001,Kanamori2012,Janocha2016}, the differentiability of squared hinge loss brings many benefits to both the computational implementation and theoretical analysis.
In this experiment, we are willing to show the similar benefits brought by the squared hinge loss.
Specifically, the training and test samples were generated via the similar way as described in {\bf Simulation 1}. Moreover, we considered different noise levels for both uniform random and outlier noise. For each case, we repeated $20$ times of experiments and record the averages of their test errors.
The test errors for different losses are presented in Table \ref{table:comp_loss_random} and Table \ref{table:comp_loss_outlier}.

From   Table \ref{table:comp_loss_random} and Table \ref{table:comp_loss_outlier}, the performance of the suggested squared hinge loss is commonly slightly better than the other three loss functions. When comparing the performance of different dictionaries, by Tables \ref{table:comp_loss_random} and  \ref{table:comp_loss_outlier}, the performance of \textit{Gauss} and \textit{Relu} are frequently better than that of \textit{polynomial} and \textit{sigmoid}, which is also observed in the previous experiments.
These show the effectiveness of the suggested squared hinge loss via comparing with the other loss functions.

\begin{table*}
\caption{Comparison on the test errors of different losses in different levels of uniform random noise. The best results among different losses are marked in bold.}
\begin{center}
\begin{tabular}{|c|c|c|c|c|c|c|c|c|c|c|c|c|}\hline
   \multirow{2}{*}{Dictionary}  & \multicolumn{4}{|c|}{20\% uniform random} & \multicolumn{4}{|c|}{30\% uniform random} & \multicolumn{4}{|c|}{40\% uniform random}\\
\cline{2-13}
                  & $\phi_{h^2}$  & $\phi_h$  &$\phi_{h^3}$  & $\phi_{sq}$  & $\phi_{h^2}$  & $\phi_h$  & $\phi_{h^3}$ & $\phi_{sq}$  & $\phi_{h^2}$   & $\phi_h$  & $\phi_{h^3}$ & $\phi_{sq}$  \\\hline
    Gauss         & {\bf 0.0239} & 0.0265 & 0.0295  & 0.0283   & {\bf 0.0418} & 0.0419 & 0.0419  & 0.0431   & {\bf 0.0851} & 0.0879 & 0.0882 & 0.0891   \\\hline
    Polynomial    & {\bf 0.0248} & 0.0265 & 0.0276  & 0.0289   & 0.0425 & {\bf 0.0403} & 0.0436  & 0.0490   & {\bf 0.0879} & 0.0929 & 0.0935 & 0.0929   \\\hline
    Sigmoid       & 0.0524 & 0.0693 & 0.0479  & {\bf 0.0271}   & 0.0597 & 0.0768 & 0.0589  & {\bf 0.0433}   & 0.0922 & 0.0925 & {\bf 0.0895} & 0.0963   \\\hline
    Relu          & {\bf 0.0219} & 0.0394 & 0.0266	& 0.0288   & {\bf 0.0335} & 0.0503 & 0.0397	 & 0.0453   & {\bf 0.0810} & 0.0864	& 0.0850 & 0.1   \\\hline
   \end{tabular}
\end{center}
\label{table:comp_loss_random}
\end{table*}

\begin{table*}
\caption{Comparison on the test errors of different losses in different levels of outlier noise with the same $tol=0.3$ and different $ratios$ varying from $0.2$ to $0.4$, where the associated noise levels are 8.51\%, 12.83\% and 17.31\%, respectively.   The best results among different losses are marked in bold.}
\begin{center}
\begin{tabular}{|c|c|c|c|c|c|c|c|c|c|c|c|c|}\hline
   \multirow{2}{*}{Dictionary} & \multicolumn{4}{|c|}{8.51\% outlier noise} & \multicolumn{4}{|c|}{12.83\% outlier noise} & \multicolumn{4}{|c|}{17.31\% outlier noise}\\
\cline{2-13}
                  & $\phi_{h^2}$   & $\phi_h$  & $\phi_{h^3}$ & $\phi_{sq}$  & $\phi_{h^2}$   & $\phi_h$  & $\phi_{h^3}$ & $\phi_{sq}$ & $\phi_{h^2}$   & $\phi_h$  & $\phi_{h^3}$ & $\phi_{sq}$ \\\hline
    Gauss         & {\bf 0.0125} & 0.0136 & 0.0126 & 0.0145  & {\bf 0.0171} & 0.0195 & 0.0255 & 0.0237   & {\bf 0.0450} & 0.0548 & 0.0498 & 0.0714   \\\hline
    Polynomial    & {\bf 0.0157} & 0.0184 & 0.0170 & 0.0172  & {\bf 0.0245} & 0.0238 & 0.0327 & 0.0332   & 0.0608 & {\bf 0.0550} & 0.0737 & 0.0866   \\\hline
    Sigmoid       & 0.0554 & 0.0584 & 0.0514 & {\bf 0.0152}  & 0.0512 & 0.0634 & 0.0487 & {\bf 0.0317}   & {\bf 0.0629} & 0.0701 & 0.0657 & 0.0770   \\\hline
    Relu          & {\bf 0.0129} & 0.0134 & 0.0149 & 0.0141  & {\bf 0.0156} & 0.0212 & 0.0211 & 0.0272   & {\bf 0.0380} & 0.0385 & 0.0405 & 0.0745
 \\\hline
   \end{tabular}
\end{center}
\label{table:comp_loss_outlier}
\end{table*}

{\bf Simulation 4: On comparison among different update schemes.}
In this experiment, we provided some comparisons between fully-corrective update  and  most of the existing types of update schemes such as that in the original boosting scheme (called \textit{OrigBoosting} for short) in \cite{Freund1997}, the regularized boosting with shrinkage (called \textit{RSBoosting} for short) in \cite{Friedman2001}, the regularized boosting with truncation (called \textit{RTBoosting}) in \cite{Zhang2005}, the forward stagewise boosting (called \textit{$\epsilon$-Boosting} for short) in \cite{Hastie2007}, and the rescaled boosting (called \textit{RBoosting} for short) suggested in the recent paper \cite{Wang2019}, when adopted to the empirical risk minimization with the squared hinge loss over the Gaussian type dictionary.
The optimal width of the Gaussian kernel was determined via cross validation from the set $\{0.1, 0.5, 1, 5\}$.
Specifically, the training and test samples were generated according to Section \ref{sc:experiment-set}, where the numbers of training and test samples were both $1,000$ and the training samples were generated with 30\% uniform random noise.
The size of the total dictionaries generated was set as $10000$.
The maximal number of iterations for the proposed FCGBoosting was set as 500, while for the other types of boosting methods, the maximal number of iterations was set as $5000$.
For each trail, we recorded the optimal test error with respect to the number of iterations, and the associated training error as well as number of dictionaries selected.
The averages of the optimal test error, training error and number of dictionaries selected over 10 repetitions are presented in Table \ref{Tab:comp-boost}.

As shown in Table \ref{Tab:comp-boost}, the performance of all boosting methods with the optimal number of dictionaries are almost the same in terms of the generalization ability measured by the test error, and under these optimal scenarios, all the boosting methods are generally well-fitted in the perspective of training error.
As demonstrated by Table \ref{Tab:comp-boost} and Fig. \ref{fig:comp-boost}(a), the most significant advantage of the adopted fully-corrective update  scheme o  is that the number of dictionaries for FCGBoosting is generally far less than that of the existing   methods   such as \textit{OrigBoosting}, \textit{RSBoosting}, \textit{RTBoosting}, \textit{$\epsilon$-Boosting} and \textit{RBoosting}.
Particularly, from Table \ref{Tab:comp-boost}, the average number of dictionaries for the proposed \textit{FCGBoosting} is only $12.6$, which is very close to the theoretical value $\left\lceil \sqrt{\frac{m}{\log m}} \right \rceil = 13$ as suggested in Theorem \ref{Thm:main-corollary}, where in this experiment $m=1000$.
Moreover, from Fig. \ref{fig:comp-boost}(b), most of the partially-corrective greedy type boosting methods select new dictionaries slowly after certain iterations, while their generalization performance improves also very slowly.

\begin{table*}
\caption{Comparisons among different types of boosting methods.}
\begin{center}
\begin{tabular}{|c|c|c|c|c|c|c|}\hline
    Boosting type       & OrigBoosting \cite{Freund1997}   & RSBoosting \cite{Friedman2001}  & RTBoosting \cite{Zhang2005}   & $\epsilon$-Boosting \cite{Hastie2007}  & RBoosting \cite{Wang2019}   & FCGBoosting (this paper)\\\hline
    Test error          & 0.0238	& 0.0256 	& 0.0239	& 0.0256	& {\bf 0.0221}	& 0.0229 \\\hline
    Training error      & 0.3071	& 0.3074 	& 0.3069	& 0.3077	& 0.3078  	& 0.3076\\\hline
    Dictionary no.      & 103.8      & 140.1     & 72.2   	& 313.1     & 120.8     & {\bf 12.6} \\\hline
   \end{tabular}
\end{center}
\label{Tab:comp-boost}
\end{table*}

\section{Real data experiments}
\label{sc:realdata}
In this section, we show the effectiveness of the proposed method via a series of experiments on 11 UCI data sets covering various areas, and an earthquake intensity classification dataset.

\subsection{UCI Datasets}

{\bf Samples.} All data is from:
\textit{https://archive.ics.uci.edu/ml/ datasets.html}.
The sizes of  data sets are listed in Table \ref{Tab:data}.
For each data set, we used $50\%$, $25\%$ and $25\%$ samples as the training, validation and test sets, respectively.

{\bf Competitors.} We evaluated the effectiveness of the proposed boosting method via comparing with the baselines and five state-of-the-art methods including
two typical support vector machine (SVM) methods with radial basis function (\textit{SVM-RBF}) and polynomial (\textit{SVM-Poly}) kernels respectively, and a fast polynomial kernel based method for classification recently proposed in \cite{Zeng2019} called \textit{FPC},
and the random forest (\textit{RF}) \cite{Breiman2001} and AdaBoost \cite{Freund1997}.
We used the well-known \textit{libsvm} toolbox to implement these SVM methods, from the website:
\textit{https://www.csie.ntu.edu.tw/~cjlin/libsvm/}.
For the proposed method, we also considered four dictionaries including \textit{Gauss}, \textit{Polynomial}, \textit{Sigmoid} and \textit{Relu}.

{\bf Implementation.}
For the proposed boosting method, we set $\alpha =1$, $\gamma=1$, the initialization $(u^0,v^0,w^0)=(0,y,0)$ and the maximal number of iterations $T=100$ for the ADMM method used in the FCG step;
the stopping criterion of the suggested method was set as the maximal iterations less than $K$, where $K$ was chosen from these five values $\left\{\left\lceil \sqrt{\frac{m}{\log m}} \right \rceil, 2 \left\lceil \sqrt{\frac{m}{\log m}} \right \rceil,\ldots, 5\left\lceil \sqrt{\frac{m}{\log m}} \right \rceil \right\}$ via cross validation;
the size of the dictionary set was set as the number of training samples $m$.
These empirical settings are generally adequate as shown in the previous simulations.

For both \textit{SVM-RBF} and \textit{SVM-Poly}, the ranges of parameters $(c,g)$ involved in \textit{libsvm} were determined via a grid search on the region $[2^{-5}, 2^5] \times [2^{-5}, 2^5]$ in the logarithmic scale,
while for \textit{SVM-Poly}, the kernel parameter was selected from the interval [1, 10] via a grid search with 10 candidates, i.e., $\{1, 2, \ldots, 10\}$.
The kernel parameter of \textit{FPC} was selected similarly to \textit{SVM-Poly}.

For \textit{RF}, the number of trees used was determined from the interval $[2, 20]$ via a grid search with 10 candidates, i.e., $\{2, 4, \ldots, 20\}$.
For \textit{AdaBoost}, the number of trees used was set as 100.
For each data set, we ran $50$ times of experiments for all algorithms, and then record their averages of test accuracies, which is defined as the percentage of the correct classified labels.

\begin{table}
\caption{Sizes of UCI data sets. In the latter tables, we use the first vocabulary of the name of the data set for short.}
\begin{center}
\begin{tabular}{|l|c|c|c|}\hline
  Data sets & Data size    & \#Attributes \\\hline
  heart                    & 270     & 14 \\ \hline
  breast\_cancer           & 683     & 9 \\ \hline
  biodeg                   & 783     & 42 \\ \hline
  banknote\_authentication & 1,372   & 4 \\ \hline
  seismic\_bumps           & 2,584   & 18  \\ \hline
  musk2                    & 6,598   & 166   \\ \hline
  HTRU2                    & 17,898  & 8  \\ \hline
  MAGIC\_Gamma\_Telescope  & 19,020  & 10 \\ \hline
  occupancy                & 20,560  & 5 \\ \hline
  default\_of\_credit\_card\_clients & 30,000  & 24  \\ \hline
  {Skin\_NonSkin}          &245,057 & 3 \\ \hline
\end{tabular}
\end{center}
 \label{Tab:data}
\end{table}

{\bf Experimental results.}
The experimental results of UCI data sets are reported in Table \ref{Tab:uci}.
From Table \ref{Tab:uci}, the proposed boosting method with different dictionaries perform slightly different.
In general, the proposed boosting method with the \textit{Gaussian}, \textit{Polynomial}, and \textit{Relu} dictionariese generally perform slightly better than the other dictionaries, as also observed in the previous experiments.
Compared to the other state-of-the-art methods, the proposed boosting method with the optimal dictionary usually performs better, where the proposed boosting method performs the best in 9 datasets, while performs slightly worse than the best results in the other 2 datasets.
If we particularly compare the performance of the proposed boosting method with the other existing methods using the same dictionary, say, \textit{Boost-Gauss vs. SVM-RBF} and \textit{Boost-Poly vs. SVM-Poly (or FPC)}, it can be observed that the adopted boosting scheme frequently improves the accuracy of these weak learners.

\begin{table*}
\caption{Test accuracies (in percentages) of different algorithms for UCI datasets, where the first four columns present the results of the proposed FCGBoosting over four differen types of dictionaries. The best and second results are marked in red and blue color, respectively.}
\begin{center}
\begin{tabular}{|c|c|c|c|c|c|c|c|c|c|c|}\hline
   {Data sets}
                   & Boost-Gauss   & Boost-Poly   & Boost-Sigmoid & Boost-ReLU  & SVM-RBF   & SVM-Poly     & FPC        & RF         & AdaBoost      & Baseline\\\hline
    heart          & 87.93         & 86.31        & 86.21         & {\color{red}89.93}       & 84.21     & {\color{blue}89.64}        & 84.14      & 89.43      & 89.44         & 81.36     \\\hline
    breast         & {\color{red}97.75}        & 96.95        & {\color{blue} 97.57}         & 97.10       & 97.19     & 96.84        & 96.78      & 96.81      & 96.34         & 96.20    \\\hline
    biodeg         & 96.86         & {\color{red}99.56}        & 98.35         & 97.71       & 96.42     & {\color{blue}99.50}        & 98.60      & 97.09      & 98.44         & 84.64
    \\\hline
    banknote       & {\color{red}100}           & {\color{red}100}          & 99.72        & {\color{blue}99.78}       & 98.07     & 97.72        & 98.15      & 98.99      & 99.17         & 95.81    \\\hline
    seismic        & {\color{red}96.44}         & {\color{red}96.44}        & {\color{red}96.44}         & {\color{red}96.44}       & 93.84     & 93.59        & 93.68      & 92.88      & {\color{blue}96.40}         & 88.00    \\\hline
    musk2          & {\color{red}100}           & 99.67        & {\color{blue}99.88}         & 99.76      & 91.11     & 92.82        & 99.08      & 96.56      & 98.85         & 90.30    \\\hline
    HTRU2          & 98.93         & 98.92        & 98.90         & {\color{red}99.00}       & 97.53     & 97.42        & 97.26      & 97.88      & 98.98         & {\color{blue}99.00}   \\\hline
    MAGIC          & 85.11         & 86.00        & 85.35         & {\color{red}87.49}       & 85.69     & 86.00        & 85.10      & {\color{blue}86.90}      & 82.67         & 86.34    \\\hline
    occupancy      & 98.80         & 98.52        & 98.51         & 98.76       & 98.63     & 98.95        & 98.77      & {\color{blue}99.14}      & {\color{red}99.55}         & 97.16
    \\\hline
    default        & {\color{blue}82.56}         & {\color{red}83.27}        & 81.07         & 82.36       & 81.60     & 82.10        & 80.51      & 81.01      & 81.67         & 82.00    \\\hline
    Skin           & 98.67         & 99.21        & 98.26         & {\color{blue}99.75}       & 98.80     & 99.06        & 98.83      & {\color{red}99.94}      & 99.14         & 98.09    \\\hline
\end{tabular}
\end{center}
 \label{Tab:uci}
\end{table*}

\subsection{Earthquake Intensity Classification}

In this experiment, we considered the \textit{U.S. Earthquake Intensity Database},
which was downloaded from:
\textit{https://www.ngdc.noaa.gov/hazard/intintro.shtml}.
This database contains more than 157,000 reports on over 20,000 earthquakes that affected the United States from 1638 through 1985.
The main features for each record in this database are the geographic latitudes and longitudes of the epicentre and  ``reporting city'' (or, locality) where the Modified Mercalli Intensity (MMI) was observed, magnitudes (as a measure of seismic energy), and the hypocentral depth (positive downward) in kilometers from the surface,
while the output label is measured by MMI, varying from 1 to 12 in integer.
An illustration of the generation procedure of each earthquake record is shown in Figure \ref{fig:earthquake}.

\begin{figure}[!t]
\begin{minipage}[b]{0.98\linewidth}
\centering
\vspace{-.5cm}
\includegraphics*[scale=.5]{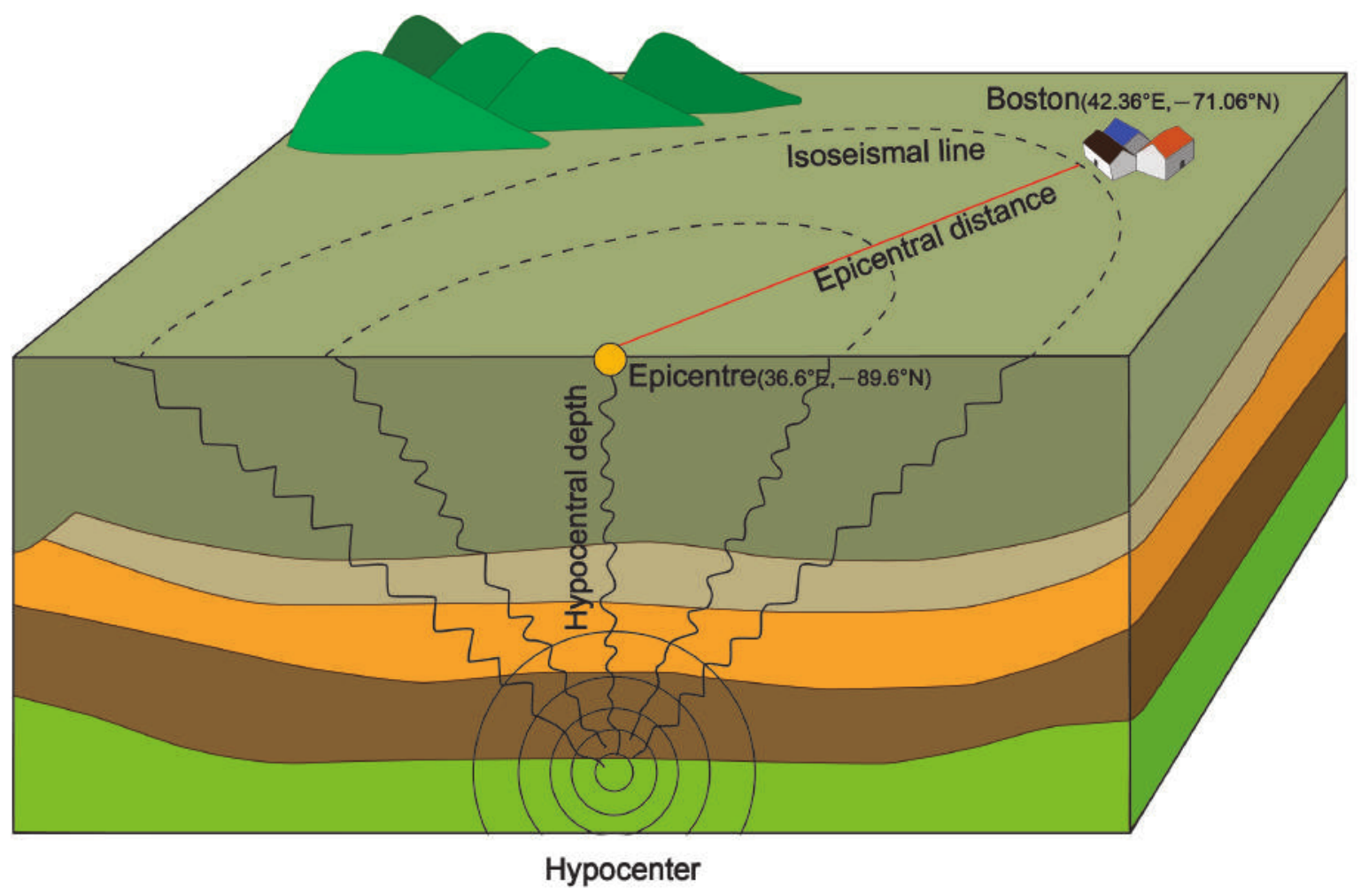}
\end{minipage}
\hfill
\caption{ An illustration of the earthquake intensity data.
}
\label{fig:earthquake}
\end{figure}

To transfer such multi-classification task into the binary classification setting considered in this paper, we set the labels lying in 1 to 4 as the positive class, while set the other labels lying in 5 to 12 as the negative class, mainly according to the damage extent of the earthquake suggested by the referred website.
Moreover, we removed those incomplete records with missing labels.
After such preprocessing, there are total 8,173 effective records.
The settings of this experiment were similar to those on the UCI data sets.
The classification accuracies of all algorithms are shown in Table \ref{Tab:earthquakedata}.

From Table \ref{Tab:earthquakedata}, the proposed boosting method with a suitable dictionary is generally better than the other state-of-the-art methods including two SVM methods, random forest, AdaBoost, and FPC.
Moreover, the performance of the proposed boosting method with the polynomial kernel in this experiment is the best one among the used dictionaries.

\begin{table*}
\caption{Test accuracies (in percentages) on the earthquake intensity data set, where the first four columns present the results of the proposed FCGBoosting over four differen types of dictionaries. The best and second results are marked in red and blue color, respectively.}
\begin{center}
\begin{tabular}{|c|c|c|c|c|c|c|c|c|c|}\hline
   {Algorithm}     & FCGBoost-Gauss   & FCGBoost-Poly   & FCGBoost-Sigmoid & FCGBoost-ReLU  & SVM-RBF   & SVM-Poly     & FPC        & RF         & AdaBoost     \\\hline
    Test Acc. (\%)  & 78.93         & {\color{red}80.48}        & 79.27        & {\color{blue}80.38}       & 80.37     & 73.92        & 80.16      & 74.51      & 75.80
    \\\hline
\end{tabular}
\end{center}
 \label{Tab:earthquakedata}
\end{table*}

\section{Proofs}
\label{sc:proofs}
In this section, we prove Theorem \ref{Thm:main-corollary} and Theorem \ref{Thm:main-Tsybakov} by developing  a novel concentration inequality associated with the squared hinge, a fast numerical convergence rate for FCGBoosting, and some standard error analysis techniques in \cite{Steinwart2008,Lin2017}.
Throughout the proofs, we will omit the subscript of $\phi_{h^2}$ for simplicity, and denote $\phi$ as the squared hinge loss.

\subsection{Concentration inequality with squared hinge loss}

Denote by
${\cal E}(f):= {\cal E}^{\phi}(f):=\int_{Z} \phi(yf(x))d \rho$ and
${\cal E}_D(f):={\cal E}_D^{\phi}(f):=\frac{1}{m}\sum_{i=1}^m \phi(y_if(x_i))$ the expectation risk and empirical risk, respectively. Let
\[f_{\rho}(x):= \arg \min_{t\in \mathbb{R}} \int_Y \phi(yt)d \rho(y|x),\]
 be the regression regression minimizing ${\cal E}(f)$. Since $\phi$ is the squared hinge loss,  it can be found in \cite{Bartlett2006} that
\begin{equation}\label{reg-relation}
    f_{\rho}(x)=2\eta(x)-1.
\end{equation}
Our aim is to derive a learning rate for the generalization error
$
     \mathcal E(f)-\mathcal E(f_\rho).
$
Noting that the squared hinge loss $\phi$ is of quadratic type,  we have
\cite[Lemma 7]{Bartlett2006} (see also \cite{Lin2017})
\begin{equation}\label{quadratic type}
     \frac12\|f-f_\rho\|_\rho^2\leq\mathcal E(f)-\mathcal
     E(f_\rho)\leq \|f-f_\rho\|_\rho^2, \quad\forall f\in L_{\rho_X}^2,
\end{equation}
where $L_{\rho_X}^2$ denotes
{the} space of    $\rho_X$ square integrable functions endowed  with
  norm $\|\cdot\|_\rho$. For   $\mathcal F\in L^1(X)$,  denote
${\cal N}_1(\epsilon,{\cal F})$ and ${\cal N}_1(\epsilon,{\cal F},x_1^m)$
as the $\varepsilon$-covering number of $\mathcal F$ under the $L^1(X)$ and $\ell^1$ norms, respectively. The following concentration inequality is the main tool in our analysis.

  \begin{theorem}\label{theorem:CONCENTRATION INEQUALITY 1}
  Let $\mathcal F$ be a set of functions $f:\mathcal X\rightarrow\mathbb R$ satisfying  $|f(x)|\leq 1, \forall x\in {\cal X}.$ Then for arbitrary $\beta>0$ and $f\in\mathcal F$, with confidence at least $1-\delta$, there holds
\begin{eqnarray}\label{con-1-1}
        && \mathcal E_D(f)-\mathcal
       E_D(f_\rho)-(\mathcal E(f)-\mathcal E(f_\rho)) \nonumber\\
       &\leq&
       \frac{17}{18}(\mathcal E(f)-\mathcal E(f_\rho))+\frac{1211}{m}\log\frac1\delta+\frac{4\beta}9 \nonumber\\
        &+&
         \frac{1164}{m }\exp\left(-\frac{\beta m}{654}\right)\mathbb{E} {\cal N}_1(\frac{\beta}{70}, {\cal F}, x_1^m).
\end{eqnarray}
\end{theorem}

It should be mentioned that a similar concentration inequality for the square loss is proved in \cite[Theorem 11.4]{Gyorfi2002}. In \cite{Wu2005}, a more general concentration inequality associated with the $L^\infty$ covering number was  presented for any bounded loss. Since we do not impose any bounded assumption on $f_{D,k}$,  it is difficult to derive an $L^\infty$ covering number estimates for the hypothesis space of FCGBoosting. Under this circumstance, a concentration inequality presented in Theorem \ref{theorem:CONCENTRATION INEQUALITY 1} is highly desired.

Let ${\cal F}$ be a set of functions $f:\mathbb{R}^d \rightarrow [-1,1]$. For $\varepsilon>0$ and $f\in\mathcal F$,
$$
     h_f(z)
     =\phi(yf(x))-\phi(yf_{\rho}(x))
 $$
 and
 $$
      v_{f,\varepsilon}(z)=\frac{h_f(z)-\mathbb{E}h_f}{\varepsilon+ \mathbb{E}h_f}.
$$
Denote
$$
   \mathcal H:=\{  h_f:f\in\mathcal F\},\qquad \mathcal V_\varepsilon:=\{
   v_{f,\varepsilon}:f\in\mathcal F\}.
$$
By the definition of $h_f$, one has
\begin{equation}\label{Eq:exp-h}
{\cal E}(f) - {\cal E}(f_{\rho})=\mathbb{E} h_f,\
 {\cal E}_D(f) - {\cal E}_D(f_{\rho})=\frac{1}{m}\sum_{i=1}^m h_f(z_i).
\end{equation}

One of the most important step-stones of our proof is the following relation between variance and expectation, which can be found in \cite[Lemma 7 and Table 1]{Bartlett2006}.
\begin{equation}\label{key-stone}
      \mathbb{E}h_f^2(z)\leq 32\mathbb{E}h_f(z)  .
\end{equation}
To derive another tool, we recall a classical concentration inequality shown in the following lemma \cite[Theorem 11.6]{Gyorfi2002}.

\begin{lemma}
\label{Lemm:concentration-ineq2}
Let  ${\cal G}$ be a set of functions $g:\mathbb{R}^d \rightarrow [-B,B]$ and $\xi, \xi_1,\ldots, \xi_m$ be i.i.d. $\mathbb{R}^d$-valued random variables. Assume $\alpha>0$, $0<\epsilon<1$, and $m\geq 1$. Then
\begin{align*}
&\mathbb{P} \left\{ \sup_{g\in {\cal G}} \frac{\frac{1}{m}\sum_{i=1}^m g(\xi_i) - \mathbb{E}g(\xi)}{\alpha+\frac{1}{m}\sum_{i=1}^m g(\xi_i)+\mathbb{E}g(\xi)} >\epsilon \right\}\\
&\leq 4\mathbb{E} {\cal N}_1(\frac{\alpha \epsilon}{5}, {\cal G}, \xi_1^m)\exp\left(-\frac{3\epsilon^2\alpha m}{40B}\right).
\end{align*}
\end{lemma}
Based on Lemma \ref{Lemm:concentration-ineq2}, we can derive the following bound for $v_{f,\varepsilon}\in \mathcal V_{\varepsilon}$ easily.
\begin{lemma}
\label{Lemm:maintool1}
For arbitrary $\beta,\varepsilon>0$ and $f\in\mathcal F$,
\begin{eqnarray*}
 &&\mathbb{E} \left[  \frac1m\sum_{i=1}^mv_{f,\varepsilon}(z_i) \right]
 \leq \frac{\mathbb{E} h_f}{3\varepsilon}  \\
   &+&
  \frac\beta{6\varepsilon}+ \frac{436}{m\varepsilon}\exp\left(-\frac{\beta m}{654}\right)\mathbb{E} {\cal N}_1(\frac{\beta}{35}, {\cal H}, z_1^m).
\end{eqnarray*}
\end{lemma}
\begin{IEEEproof}
Since $|h_f(z)|\leq 4$,  for $\varepsilon=1/7$,
it follows from Lemma \ref{Lemm:concentration-ineq2} with¡¡ $\mathcal G=\mathcal H$ and $B=4$
that with confidence
 $$
1- 4\mathbb{E} {\cal N}_1(\frac{\alpha }{35}, {\cal H}, z_1^m)\exp\left(-\frac{3 \alpha m}{1960}\right),
 $$
for all $h_f\in \mathcal H$, there holds
\begin{eqnarray*}
    \frac{1}{m}\sum_{i=1}^m 6h_f(z_i) - 8\mathbb{E}h_f(z)  \leq \alpha.
\end{eqnarray*}
For arbitrary $\beta\geq 0$ and $f\in\mathcal F$, if $  \frac{6}{m}\sum_{i=1}^m h_f(z_i) - 8 \mathbb{E}h_f(z) \geq0$,
we
 apply the formula
\begin{equation}\label{Expectedform}
  \mathbb E[\xi] =\int_0^\infty \mathbb P\left[\xi > t\right] d t
\end{equation}
and obtain
\begin{eqnarray*}
     && \mathbb{E} \left[ \frac{6}{m}\sum_{i=1}^m h_f(z_i) - 8 \mathbb{E}h_f(z)\right]
     \leq \beta \nonumber \\
     &+&
     4\int_{\beta}^\infty  \mathbb{E} {\cal N}_1(\frac{\alpha }{35}, {\cal H}, z_1^m)\exp \left(-\frac{3 \alpha m}{1960}\right)d\alpha \nonumber \\
     &\leq&
     \beta+4\mathbb{E} {\cal N}_1(\frac{\beta}{35}, {\cal H}, z_1^m)
      \int_{\beta}^\infty \exp\left(-\frac{3 \alpha m}{1960}\right)d\alpha \nonumber \\
      &\leq&
      \beta+ \frac{2616}{m}\exp\left(-\frac{\beta m}{654}\right)\mathbb{E} {\cal N}_1(\frac{\beta}{35}, {\cal H}, z_1^m).
\end{eqnarray*}
If $  \frac{6}{m}\sum_{i=1}^m h_f(z_i) - 8 \mathbb{E}h_f(z) <0$, the above estimate also holds trivially.
Then for arbitrary $\varepsilon>0$ and $f\in\mathcal F$, it follows from the above estimate that \begin{eqnarray*}
   &
   \mathbb E[\frac{1}{m}\sum_{i=1}^m v_{f,\varepsilon}(z_i)]\leq\frac{2\mathbb E[h_f]+ \mathbb{E} \left[ \frac{6}{m}\sum_{i=1}^m h_f(z_i) -8  \mathbb{E}h_f(z)\right]}{6\varepsilon}\\
   &\leq
   \frac1{6\varepsilon}\left(2 \mathbb E h_f+\beta+ \frac{2616}{m}\exp\left(-\frac{\beta m}{654}\right)\mathbb{E} {\cal N}_1(\frac{\beta}{35}, {\cal H}, z_1^m)\right).
\end{eqnarray*}
This completes the proof of Lemma \ref{Lemm:maintool1}.
\end{IEEEproof}

The third tool of our proof is a simplified Talagrand's inequality, which can be easily deduced from Theorem 7.5 and Lemma 7.6 in \cite{Steinwart2008}.

\begin{lemma}\label{Lemma:Talagrand inequality}
Let $\varepsilon>0$, $B\geq0$
and $\sigma\geq0$ be constants such that
$\mathbb E[v_{f,\varepsilon}^2]\leq\sigma^2$  and $\|v_{f,\varepsilon}\|_\infty\leq B$ for all $v_{f,\epsilon}\in\mathcal
V_{\epsilon}$.
Then, for any $\tau>0$,  $\gamma>0$ and $f\in\mathcal F$, with confidence $1-e^{-\tau}$, there holds
\begin{eqnarray*}
       &   \sup_{ v_{f,\varepsilon} \in \mathcal
        V_{\varepsilon}} \frac1m\sum_{i=1}^mv_{f,\varepsilon}(z_i)
        \leq
        \sqrt{\frac{2\tau\sigma^2}{m}}+\left(\frac23+\frac1\gamma\right)
        \frac{\tau B}{m} \\
        &+(1+\gamma) \mathbb{E} \left[\sup_{ v_{f,\varepsilon} \in \mathcal
        V_{\varepsilon}} \frac1m\sum_{i=1}^mv_{ f,\varepsilon}(z_i)\right].
\end{eqnarray*}
\end{lemma}

With these tools, we are now in the position to prove Theorem \ref{theorem:CONCENTRATION INEQUALITY 1}.
\begin{IEEEproof}
 For
arbitrary $f\in \mathcal F$, we have
  $\|h_f\|_\infty\leq 4$ and $\|h_f-\mathbb{E}h_f\|_\infty\leq 8$.
   Let
$\varepsilon\geq2\inf_{f\in \mathcal F}\mathbb E[h_f]$. Then for arbitrary
$v_{f,\varepsilon}\in\mathcal V_{\varepsilon}$, there
exists an $f\in \mathcal F$ such that $
v_{f,\varepsilon}=\frac{ h_f-\mathbb E[h_f]}{\mathbb E[h_f]+\varepsilon}$.
Then, we get from \eqref{key-stone} that
\begin{equation}\label{1.con.1}
     \|v_{f,\varepsilon}\|_\infty\leq\frac{8}{\varepsilon}=:B,
\end{equation}
and
\begin{equation}\label{1.con.2}
     \mathbb E[v^2_{f,\varepsilon}]\leq
     \frac{\mathbb E[h_f^2]}{(\mathbb E[h_f]+\varepsilon)^2}
     \leq  \frac{32\mathbb E[h_f]}{(\mathbb E[h_f]+\varepsilon)^2} \leq \frac{16}{\varepsilon}.
\end{equation}
Then Lemma
\ref{Lemma:Talagrand inequality} with $\gamma=1/3$ and $\varepsilon\geq \inf_{ f\in
\mathcal F }\mathbb E[h_f]$,   Lemma \ref{Lemm:maintool1},   (\ref{1.con.1})
and (\ref{1.con.2}) that with confidence at least $1-e^{-\tau}$, for any $f\in\mathcal F$,
there holds
\begin{eqnarray*}
        && \frac1m\sum_{i=1}^mv_{f,\varepsilon}(z_i)
        \leq
        \sqrt{\frac{32\tau}{m\varepsilon}}+\frac{88\tau}{3m\varepsilon}
        +\frac4{9\varepsilon}\mathbb E[h_f]\\
        &+&
           \frac{2\beta}{9\varepsilon}+ \frac{582}{m\varepsilon}\exp\left(-\frac{\beta m}{654}\right)\mathbb{E} {\cal N}_1(\frac{\beta}{35}, {\cal H}, z_1^m).
\end{eqnarray*}
For arbitrary $f\in \mathcal F$, set $\tau=\log\frac1\delta$ and
$\varepsilon= \mathcal E(f)-\mathcal E(f_\rho) \geq \inf_{f\in \mathcal
F}\mathbb E[h_f]$. It follows from \eqref{Eq:exp-h} that , with
confidence $1-\delta$, there holds
\begin{eqnarray*}
        && \mathcal E_D(f)-\mathcal
       E_D(f_\rho)-(\mathcal E(f)-\mathcal E(f_\rho))\\
       &\leq&
       \frac8{9 }\mathbb E[h_f]+\sqrt{\frac{128(\mathcal E(f)-\mathcal E(f_\rho))}{m}\log\frac1\delta}+\frac{176}{3m}\log\frac1\delta\\
        &+&
        \frac{4\beta}9+ \frac{1164}{m }\exp\left(-\frac{\beta m}{654}\right)\mathbb{E} {\cal N}_1(\frac{\beta}{35}, {\cal H}, z_1^m).
\end{eqnarray*}
For arbitrary $h_{f_1},h_{f_2}\in\mathcal H$, we have
\begin{eqnarray*}
\|h_{f_1}-h_{f_2}\|_{\ell^1}=\|\phi(yf_1)- \phi(yf_2)\|_{\ell^1}
   \leq
   2\|f_1-f_2\|_{\ell^1},
\end{eqnarray*}
Then,
$$
     \mathbb{E} {\cal N}_1(\frac{\beta}{35}, {\cal H}, z_1^m)
     \leq
     \mathbb{E} {\cal N}_1(\frac{\beta}{70}, {\cal F}, x_1^m).
$$
Noting further that the
 element inequality $\sqrt{ab}\leq\frac12(a+b)$ for
$a,b>0$ yields
$$
  \sqrt{\frac{64(\mathcal E(f)-\mathcal E(f_\rho))}{m}\log\frac1\delta}
  \leq \frac1{18}(\mathcal E(f)-\mathcal E(f_\rho))+\frac{1152}{m}\log\frac1\delta.
$$
This proves (\ref{con-1-1}) and  completes the proof of Theorem
\ref{theorem:CONCENTRATION INEQUALITY 1}.
\end{IEEEproof}

\subsection{Numerical convergence without boundedness assumption}
In this part, we show the fast convergence rate of FCGBoosting without imposing any boundedness assumption. Our proof is motivated by  \cite{Shalev-Shwartz2010} by taking the special property of the squared hinge.  The following numerical convergence rate is another main tool in our proof.

\begin{proposition}
\label{Proposition:hypothesis-error}
For arbitrary $h\in \mathrm{span}{\cal G}_n$,  we have
\begin{align}
\label{Eq:hypothesis-error-bound}
\mathcal E_D(f_{D,k})-\mathcal E_D(h) \leq \frac{4\|h\|_{\ell^1}^2}{k}.
\end{align}
\end{proposition}

It can be found in \cite{Barron2008} and
Proposition \ref{Proposition:hypothesis-error} that the numerical convergence rates for FCGBoosting are the same for the square loss and squared hinge loss. To prove the above proposition,
we need the following lemma, which was proved in \cite[Lemma B.2]{Shalev-Shwartz2010}.
\begin{lemma}
\label{Lemm:sequence-ineq}
Let $c>0$ and let $\gamma_0, \gamma_1, \ldots $ be a sequence such that $\gamma_{t+1} \leq \gamma_t - c \gamma_t^2$ for all $t$. Let $\epsilon$ be a positive scalar and $k$ be a positive integer such that $k \geq \lceil \frac{1}{c\epsilon}\rceil$. Then $\gamma_k \leq \epsilon$.
\end{lemma}

Based on this lemma, we estimate the upper bound of ${\cal H}(D,k,h)$ in the following proposition. Similar results can be found in \cite[Theorem 2.7]{Shalev-Shwartz2010} for the general smooth type loss functions. We provide its proof here for the sake of completeness.

\begin{proof}[Proof of Proposition \ref{Proposition:hypothesis-error}]
Let $h = \sum_{j=1}^n \alpha_j g_j$ be an arbitrary function in $\mathrm{span}{\cal G}_n$. For $\alpha=(\alpha_1,\dots,\alpha_n)^T$, let $V:=V_{\alpha}:= \mathrm{supp}(\alpha)$ be the support of $\alpha$, which implies
\begin{align}
\label{Eq:h-coeff}
\alpha_j = 0, \ j\in V^c.
\end{align}
By Algorithm \ref{alg:boosting},
\begin{align}
{\cal E}_D(f_{D,k+1})
&= \min_{\mathrm{supp}(f) = {\cal T}^{k+1}} {\cal E}_D(f) \nonumber\\
&\leq \min_{j\in V}\min_{\eta} {\cal E}_D(f_{D,k} + \eta \cdot \mathrm{sgn}(h_j) g_j). \label{Eq:Err-k+1}
\end{align}
By the Lipschitz continuity of $\phi'$ with the Lipschtz constant $L=2$ and $\|g_j\|_\infty\leq 1$, \cite[Lemma B.1]{Shalev-Shwartz2010} (see also\cite{Nestrov2004}) shows
\begin{align*}
&{\cal E}_D(f_{D,k} + \eta \mathrm{sgn}(h_j)g_j)\\
&\leq {\cal E}_D(f_{D,k}) + \eta \mathrm{sgn}(h_j) \langle \nabla {\cal E}_D(f_{D,k}), g_j \rangle + \eta^2\\
&=: E_j(\eta)
\end{align*}
for any $\eta \in \mathbb{R}$ and $j\in V$. Let $s:= \sum_{j\in V} |\alpha_j| = \|h\|_{\ell^1}$,
then
\begin{align}
&s \min_{j\in V } E_j(\eta)
\leq \sum_{j\in V } |\alpha_j| E_j(\eta) \nonumber\\
&= s {\cal E}_D(f_{D,k}) + \eta \sum_{j\in V } \alpha_j \langle \nabla {\cal E}_D(f_{D,k}), g_j \rangle + s \eta^2. \label{Eq:sEj}
\end{align}
Since ${\cal E}_D(f_{D,k}) = \min_{\mathrm{supp}(f) = {\cal T}^k} {\cal E}_D(f)$, and $f_{D,k}=\sum_{j=1}^n \beta_j^k g_j$, there holds
\begin{align}
\label{Eq:opt-cond}
\langle \nabla {\cal E}_D(f_{D,k}), g_j \rangle =0, \ \forall j\in {\cal T}^k,
\end{align}
and $\beta_j^k =0$ for any $j\in ({\cal T}^k)^c$.
Thus,
\begin{align}
\sum_{j\in V } \alpha_j \langle \nabla {\cal E}_D(f_{D,k}), g_j \rangle
&=\sum_{j\in V \backslash {\cal T}^k} \alpha_j \langle \nabla {\cal E}_D(f_{D,k}), g_j \rangle \nonumber\\
&=\sum_{j\in V \backslash {\cal T}^k} (\alpha_j-\beta_j^k) \langle \nabla {\cal E}_D(f_{D,k}), g_j \rangle \nonumber\\
&=\sum_{j\in V \cup {\cal T}^k} (\alpha_j-\beta_j^k) \langle \nabla {\cal E}_D(f_{D,k}), g_j \rangle \nonumber\\
&=\langle \nabla {\cal E}_D(f_{D,k}), h-f_{D,k} \rangle, \label{Eq:inner-product}
\end{align}
where the first equality holds for \eqref{Eq:opt-cond},  the second equality holds for $\beta_j^k = 0, \forall j\in ({\cal T}^k)^c$,  the third equality holds for $\langle \nabla {\cal E}_D(f_{D,k}), g_j\rangle=0, \forall j\in {\cal T}^k$, and the final equality holds for $\alpha_j = \beta_j^k =0, \forall j\in (V\cup {\cal T}^k)^c$.
Furthermore, by the convexity of $\phi$, there holds
\begin{align}
\label{Eq:convex-ineq}
{\cal E}_D(h) - {\cal E}_D(f_{D,k}) \geq \langle \nabla {\cal E}_D(f_{D,k}), h-f_{D,k}\rangle.
\end{align}
Thus, by \eqref{Eq:Err-k+1}, \eqref{Eq:sEj}, \eqref{Eq:inner-product} and \eqref{Eq:convex-ineq}, there holds
\begin{align*}
&s {\cal E}_D(f_{D,k+1}) \leq s {\cal E}_D(f_{D,k})
-\eta ({\cal E}_D(f_{D,k}) - {\cal E}_D(h)) + s\eta^2
\end{align*}
for any $\eta \in \mathbb{R}$.
Taking $\eta = \frac{{\cal E}_D(f_{D,k}) - {\cal E}_D(h)}{2s}$, the above inequality yields
\begin{align*}
{\cal E}_D(f_{D,k+1})
&\leq {\cal E}_D(f_{D,k}) - \frac{({\cal E}_D(f_{D,k}) - {\cal E}_D(h))^2}{4\|h\|_{\ell^1}^2}.
\end{align*}
Denote $\epsilon_k = {\cal E}_D(f_{D,k}) - {\cal E}_D(h)$. The above inequality implies
\begin{align*}
\epsilon_{k+1} \leq \epsilon_k - \frac{\epsilon_k^2}{4\|h\|_{\ell^1}^2}.
\end{align*}
Then Proposition \ref{Proposition:hypothesis-error} follows from    Lemma \ref{Lemm:sequence-ineq}.
\end{proof}

\subsection{Learning rate analysis}
Based on Theorem \ref{theorem:CONCENTRATION INEQUALITY 1} and Proposition \ref{Proposition:hypothesis-error}, we are in a position to prove the following theorem, which is a key stone to prove Theorems \ref{Thm:main-corollary} and \ref{Thm:main-Tsybakov}.

\begin{theorem}
\label{Thm:main}
Let ${\cal G}_n := \{g_j\}_{j=1}^n$ be a set of dictionaries with $n \sim m^a$ for some $a\geq 1$.
Let $f_{D,k}$ be the predictor of Algorithm \ref{alg:boosting} after $k$ running iterations.
If Assumption \ref{assump:approx-error} holds, then for all $k\in \mathbb{N}$, with confidence at least $1-\delta$, there holds
\begin{align}\label{oracle}
     {\cal E}(\pi f_{D,k}) - {\cal E}(f_{\rho})
    \leq  \tilde{C}\left(n^{-2r}+k^{-1}
     +
  \frac{k\log m}{m}\log\frac2\delta\right).
\end{align}
where $\pi t=\min\{1,|t|\} \cdot \mathrm{sgn}(t)$ denotes the truncation of $t\in \mathbb R$ to $[-1,1]$ and $\tilde{C}$ is a constant depending only on $C_1$, $C_2$ and $a$, whose concrete value will be given in the proof.
\end{theorem}

 The proof idea is somewhat standard \cite{Lin2013,Lin2017} which devotes to decomposing the generalization error into three terms: approximation error, sample error and hypothesis error.  The approximation error can be derived from Assumption \ref{assump:approx-error} and (\ref{quadratic type}), the hypotheses error can be deduced from Proposition \ref{Proposition:hypothesis-error} and the sample error error can be derived using the concentration inequality in Theorem \ref{theorem:CONCENTRATION INEQUALITY 1} and a covering number estimate of the hypothesis space.

 Given any set $\Lambda \subset {\cal G}_n,$ we define ${\cal J}_{\Lambda}:= \mathrm{span}\{g: g\in \Lambda\}$ and denote by $\pi {\cal J}_{\Lambda}:= \{\pi f: f\in {\cal J}_{\Lambda}\}$ the set of all truncations of the elements of ${\cal J}_{\Lambda}$. We then define
\begin{align}
\label{Eq:funcation-space}
{\cal F}_k := \bigcup_{\Lambda \subset {\cal G}_n, \sharp (\Lambda)\leq k} \pi {\cal J}_{\Lambda},
\end{align}
where $\sharp (\Lambda)$ represents the cardinality of set $\Lambda$ and
  $\pi f(x) = \min\{1,|f(x)|\}\cdot\mathrm{sgn}(f(x))$ is the truncation operator on $f(x)$ to $[-1,1]$.
Due to \cite[Lemma 3.3]{Barron2008} with $B=1$, there holds
\begin{align*}
{\cal N}_1(\epsilon,{\cal F}_k) \leq 3 m^{ak}\left(\frac{2e}{\epsilon}\log \frac{3e}{\epsilon} \right)^{k+1}.
\end{align*}
Noting further $\sup_{x_1^m} {\cal N}_1(\epsilon,{\cal F}_k,x_1^m) \leq {\cal N}_1(\epsilon,{\cal F}_k)$, we have the following covering number estimates.

\begin{lemma}
\label{Lemm:covering-number}
Let ${\cal F}_k$ be defined as in \eqref{Eq:funcation-space}. Assume that $n \sim m^a$ for some $a\geq 1$.
Then for any probability measure $v$, for any $\epsilon>0$, we have the following $L_1$ empirical covering number estimate of ${\cal F}_k$
\begin{align}
\label{Eq:covering-number}
\sup_{x_1^m} {\cal N}_1(\epsilon,{\cal F}_k,x_1^m) \leq 3 m^{ak}\left(\frac{2e}{\epsilon}\log \frac{3e}{\epsilon} \right)^{k+1}.
\end{align}
\end{lemma}

With these helps, we prove Theorem \ref{Thm:main} as follows.
\begin{IEEEproof}[Proof of Theorem \ref{Thm:main}]
 Since $y=\{-1,1\}$, we have
\begin{equation}\label{Truncation}
      {\cal E}_D(\pi f_{D,k}) \leq {\cal E}_D(f_{D,k}).
\end{equation}
Then for arbitrary $h \in \mathrm{span}{\cal G}_n$, there holds
\begin{align}
\label{Eq:error-decomposition}
& {\cal E}(\pi f_{D,k}) - {\cal E}(f_{\rho}) \\
&\leq {\cal A}({\cal G}_n,h) + {\cal H}(D,k,h) + {\cal S}_D(h) - S_D(\pi f_{D,k}), \nonumber
\end{align}
where
\begin{align*}
& {\cal A}({\cal G}_n,h):= {\cal E}(h) - {\cal E}(f_{\rho}), \\
& {\cal H}(D,k,h) := {\cal E}_D(f_{D,k}) - {\cal E}_D(h),\\
& {\cal S}_D(f) := ({\cal E}_D(f) - {\cal E}_D(f_{\rho})) - ({\cal E}(f) - {\cal E}(f_{\rho}))
\end{align*}
are the approximation, hypothesis and sample errors, respectively.
 Due to Assumption \ref{assump:approx-error}, \eqref{reg-relation}, \eqref{quadratic type} and $\|f\|_\rho\leq\|f\|_\infty$, we can get the following approximation error estimate directly.
\begin{align}
\label{Eq:approx-error-bound}
{\cal A}({\cal G}_n, h_0) \leq C_1^2n^{-2r}, \ \text{and} \ \|h_0\|_{\ell^1} \leq C_2.
\end{align}
Furthermore, Proposition \ref{Proposition:hypothesis-error} shows
\begin{align}
\label{Eq:hypothesis-error-bound-1-1}
\mathcal H(D,k,h) \leq \frac{4\|h\|_{\ell^1}^2}{k}.
\end{align}
 The only thing remainder is to bound the sample error ${\cal S}_D(h)$ and $- S_D(\pi f_{D,k})$. The former is pretty standard, we refer the readers to \cite{Shi2011,Lin2017}  (with a slight change of constant) that with confidence $1-\delta/2$, there holds,
 \begin{align}
\label{Eq:sample-error1-bound}
{\cal S}_D(h)
&\leq \frac{4(B_0+1)^2\log \frac{2}{\delta}}{3m} \\
&+ 2(B_0+1) \|h-f_{\rho}\|_{\rho}\sqrt{\frac{2\log \frac{2}{\delta}}{m}}, \nonumber
\end{align}
where $B_0 := \max \{\|h\|_{\infty},1\}$. Now, we turn to bound $- S_D(\pi f_{D,k})$. From Theorem \ref{theorem:CONCENTRATION INEQUALITY 1} with $\mathcal F=\mathcal F_k$, we  have for arbitrary $\beta>0$, with confidence $1-\delta/2$
 \begin{align}
         - S_D(\pi f_{D,k})
       &\leq
       \frac{17}{18}(\mathcal E(\pi f_{D,k})-\mathcal E(f_\rho))+\frac{1211}{m}\log\frac2\delta \label{sample-err-2-1}\\
       & +\frac{4\beta}9 +
         \frac{1164}{m }\exp\left(-\frac{\beta m}{654}\right)\mathbb{E} {\cal N}_1(\frac{\beta}{70}, {\cal F}_k, x_1^m). \nonumber
\end{align}
For $\beta\geq1/m$, Lemma \ref{Lemm:covering-number} implies
\begin{eqnarray*}
  &&\frac{1164}{m }\exp\left(-\frac{\beta m}{654}\right)\mathbb{E} {\cal N}_1(\frac{\beta}{70}, {\cal F}_k, x_1^m)\\
  &\leq&
  \exp\left\{\log\left[3492m^{ak-1}(140e m\log (210e m))^{k+1}\right]-\frac{\beta m}{654}\right\}\\
  &\leq&
  \exp\bigg\{ \log(3492)+(ak-1)\log m+(k+1)\Big[\log(140e)\\
  &+&\log m+\log\log(210em)\Big]-\frac{\beta m}{654}\bigg\}\\
  &\leq&
  \exp\Big\{\bar{C}ak\log m-\frac{\beta m}{654}\Big\},
\end{eqnarray*}
 where $\bar{C}\geq 1$ is an absolute constant. Setting
 $\beta=\frac{1308\bar{C}a k\log m}{m},$ we obtain from \eqref{sample-err-2-1} that
 \begin{eqnarray}\label{sample-err-2}
         - S_D(\pi f_{D,k})
       \leq
       \frac{17}{18}(\mathcal E(f)-\mathcal E(f_\rho))
         +
         \frac{\bar{C}_1 a k\log m}{m}\log\frac2\delta  ,
\end{eqnarray}
 where we use
 $$
    \exp\big\{-\bar{C}a k\log m\big\}\leq m^{-a \bar{C}k}\leq \frac1m
 $$
  and $\bar{C}_1$ is an absolute constant.
 Plugging  \eqref{sample-err-2}, \eqref{Eq:sample-error1-bound}, \eqref{Eq:hypothesis-error-bound-1-1} and
 \eqref{Eq:approx-error-bound} into (\ref{Eq:error-decomposition})  and noting $\|h\|_\infty\leq\|h\|_{\ell^1}\leq C_2$, $\sqrt{ab}\leq \frac12(a+ b)$, we obtain that
 \begin{eqnarray*}
  &&
    \frac1{18}({\cal E}(\pi f_{D,k}) - {\cal E}(f_{\rho}) )
    \leq  C_1^2n^{-2r}+\frac{4C_2^2}{k}\\
    &+&
    \frac{4(C_2+2)^2\log \frac{2}{\delta}}{3m}
  +  (C_2+2)( C_1^2 n^{-2r}+ 2m^{-1}\log \frac{2}{\delta} )\\
  &+&
  \frac{\bar{C}_1a k\log m}{m}\log\frac2\delta
\end{eqnarray*}
  holds with confidence $1-\delta$. That is, with confidence $1-\delta$, there holds
   \begin{eqnarray*}
    {\cal E}(\pi f_{D,k}) - {\cal E}(f_{\rho})
    \leq  \tilde{C}\left(n^{-2r}+k^{-1}
     +
  \frac{k\log m}{m}\log\frac2\delta\right)
\end{eqnarray*}
where $\tilde{C}=18\max\{(C_2+3)C_1^2,4C_2^2,2(C_2+2)^2+2(C_2+2)+a\bar{C}_1\}$.
This completes the proof of Theorem \ref{Thm:main}.
\end{IEEEproof}

Based on Theorem \ref{Thm:main}, we can prove Theorems \ref{Thm:main-corollary} and \ref{Thm:main-Tsybakov} as follows.

\begin{proof}[Proof of Theorem \ref{Thm:main-corollary}]
By \eqref{oracle}, if $n \sim m^a$ for $a\geq 1$, $r\geq \frac{1}{4a}$, and $k\sim \sqrt{\frac{m}{\log m}}$, then
\begin{align}
\label{Eq:generalization-error}
&{\cal E}(\pi f_{D,k}) - {\cal E}(f_{\rho})\leq C_4 \left(\frac{m}{\log m}\right)^{-1/2} \log \frac{4}{\delta},
\end{align}
where $C_4$ is a positive constant independent of $\delta$ or $m$.
Furthermore, by the comparison inequality established by \cite{Chen2004}, that is,
\begin{align}
\label{Eq:comparison-ineq1}
{\cal R}(\mathrm{sgn}(f)) - {\cal R}(f_c) \leq C_{\phi}\sqrt{{\cal E}(f) - {\cal E}(f_{\rho})}
\end{align}
for some constant $C_{\phi}>0$. Thus, Theorem \ref{Thm:main-corollary} follows from \eqref{Eq:generalization-error} and \eqref{Eq:comparison-ineq1}.
\end{proof}

\subsubsection{Proof of Theorem \ref{Thm:main-Tsybakov}}
\label{sc:proof-Theorem2}

\begin{proof}
The claim of this theorem is yielded by \eqref{Eq:generalization-error} and the comparison inequality under Assumption \ref{assump:Tsybakov-noise} (\cite{Bartlett2006}, \cite{Xiang2011}, see also \cite[Theorem 8.29]{Steinwart2008}), saying that for arbitrary measurable function $f:X \rightarrow \mathbb{R}$, there holds
\begin{align*}
&{\cal R}(\mathrm{sgn}(f)) - {\cal R}(f_c) \\
&\leq 2^{\frac{3q+4}{q+2}}(\hat{c}_q)^{-\frac{q}{q+2}}C_{\phi,1}^{-\frac{q+1}{q+2}}\left({\cal E}(f) - {\cal E}(f_{\rho})\right)^{\frac{q+1}{q+2}}, 
\end{align*}
where $C_{\phi,1}$ is a constant depending only on the loss $\phi$.
Let $C_5 = 2^{\frac{3q+4}{q+2}}(\hat{c}_q)^{-\frac{q}{q+2}}C_{\phi,1}^{-\frac{q+1}{q+2}} C_3^{\frac{q+1}{q+2}}$. This finishes the proof.
\end{proof}

\section{Conclusion}
\label{sc:conclusion}

Binary classification is a very significant problem in machine learning.
In this paper, we propose an efficient boosting method, aiming to improve the classification accuracy and establish the theoretical generalization guarantee.
We adopt the fully-corrective greedy update scheme to the boosting procedure, and then exploit the special form of the so-called \textit{squared hinge} loss to establish its fast learning rates in the framework of statistical learning, under some regular assumptions.
Certain efficient early stopping rule is also derived for the proposed boosting method.
The toy simulations are implemented to verify the feasibility of the proposed method as well as our theoretical findings.
Moreover, a series of UCI data experiments and a real earthquake intensity data experiment are provided to show the effectiveness of the proposed method, particularly, the classification accuracies can be improved via our proposed method with an appropriate dictionary.
Some future work is how to adopt the tree structures as the dictionary into the proposed method and establish the associated theoretical generalization guarantees.

%
%

\section*{Acknowledge}
The authors would like to thank Mr. Huaqing Zhang from China Academy of Art for his kindly help to improve the quality of Figure \ref{fig:earthquake}.


\section*{Appendix}

\subsection*{Appendix A: ADMM for FCG subproblem with squared hinge}
\label{sc:append-ADMM}
Note that in Algorithm \ref{alg:boosting}, the fully-corrective greedy  step \eqref{FCB} presented in the functional form is equivalent to the following optimization problem presented in the vector form, that is,
\begin{align}
\label{Eq:FCG-subproblem}
\beta^* =&\arg\min_{\beta \in \mathbb{R}^n} \frac{1}{m}\sum_{i=1}^m \left(1-y_i \sum_{j=1}^n \beta_j g_j(x_i) \right)_+^2 \\
&\text{subject to} \ \  \mathrm{supp}(\beta) \subset {\cal T}^{k}, \nonumber
\end{align}
where $(z)_+ := \max \{0,z\}$ for any $z\in \mathbb{R}$, and $\mathrm{supp}(\beta)$ denotes the support set of $\beta$, i.e., the nonzero set of $\beta$.
Then $f_{D,k} = \sum_{j\in {\cal T}^{k}} \beta_j^* g_j$.

In the following, we describe how to adopt the alternating direction method of multipliers (ADMM) to fast solve the optimization problem \eqref{Eq:FCG-subproblem}.
Let $s$ be the cardinality of the set ${\cal T}^{k}$, $u = \beta_{{\cal T}^{k}}$, $A \in \mathbb{R}^{m\times s}$ be a matrix induced by the input $\{x_i\}_{i=1}^m$ and the dictionaries selected in ${\cal T}^{k}$, i.e., $A_{ij} = g_{{\cal T}^{k}(j)}(x_i)$, where ${\cal T}^{k}(j)$ represents the $j$-th component of the set ${\cal T}^{k}$ and $g_{{\cal T}^{k}(j)}$ represents the dictionary with the index ${\cal T}^{k}(j)$.
Thus, the optimization problem \eqref{Eq:FCG-subproblem} can be reformulated as the following,
\begin{align}
\label{Eq:opt_prob}
u^* = \arg\min_{u\in \mathbb{R}^s} \frac{1}{m}\sum_{i=1}^m \left(1-y_i \sum_{j=1}^s A_{ij}u_j\right)_+^2.
\end{align}

First, we reformulate \eqref{Eq:opt_prob} as the following equivalent problem
\begin{align}
\label{Eq:opt_prob_admm}
&\mathop{\mathrm{minimize}}_{u \in \mathbb{R}^s, v\in \mathbb{R}^m} \quad f(v) := \frac{1}{m}\sum_{i=1}^m \left(1-y_i v_i\right)_+^2 \\
& \text{subject to} \quad v = Au. \nonumber
\end{align}
Then its augmented Lagrangian function is
\begin{align}
\label{Eq:ALM}
{\cal L}_{\gamma}(u,v,w) = f(v) + \langle w, v-Au \rangle + \frac{\gamma}{2} \|v-Au\|_2^2,
\end{align}
where $w\in \mathbb{R}^m$ is a multiplier variable, $\gamma>0$ is an augmented parameter.

Based on the above defined augmented Lagrangian function, the ADMM method for \eqref{Eq:opt_prob} can be described as follows:
given the initialization $u^0, v^0, w^0$, for $t=1, 2, \ldots,$\\
\textbf{(a) update $u^t$ via proximal scheme:} for some $\alpha>0$ (default $\alpha=1$),
\begin{align}
\label{Eq:update-u-prox}
u^t = \arg\min_{u \in \mathbb{R}^s} {\cal L}_{\gamma}(u,v^{t-1},w^{t-1}) + \frac{\alpha}{2} \|u-u^{t-1}\|_2^2,
\end{align}
which implies that
\begin{align}
\label{Eq:update-u-alg2}
u^t = (\gamma A^TA+\alpha {\bf I})^{-1} \left(A^T(\gamma v^{t-1} + w^{t-1}) + \alpha u^{t-1} \right).
\end{align}
\\
\textbf{(b) update $v^t$:}
\begin{align}
\label{Eq:update-v-prox}
v^t = \arg\min_{v \in \mathbb{R}^m}\ {\cal L}_{\gamma}(u^t,v,w^{t-1}).
\end{align}
From the above equation, the $v^t$-subproblem is separable and thus can be reduced to the following univariate optimization problem,
\begin{align}
\label{Eq:update-v-alg2}
&v_i^t = \\
&\arg \min_{v_i \in \mathbb{R}} \left(1-y_i v_i\right)_+^2 + \frac{m \gamma}{2}\left(v_i - \sum_{j=1}^n A_{ij}u_i^{t} + \gamma^{-1}w_i^{t-1}\right)^2, \nonumber
\end{align}
which has the closed form solution as shown in Lemma \ref{Lemm:solution-squaredhinge-min} in Appendix B.

In a summary, the specific procedure of ADMM for the FCG optimization is presented in Algorithm \ref{alg:ADMM}.

\begin{algorithm}\caption{ADMM for FCG subproblem \eqref{Eq:FCG-subproblem}}\label{alg:ADMM}
\begin{algorithmic}
\STATE{ {\bf Input}: training sample set $D:= \{x_i,y_i\}_{i=1}^m$, $X:=(x_1,\ldots, x_m)^T$, $y:=(y_1,\ldots,y_m)^T$, dictionaries $\{g_j\}_{j\in {\cal T}^k}$, and matrix $A\in \mathrm{R}^{m\times s}$ with $A_{ij} = g_{{\cal T}^{k}(j)}(x_i)$, where $s:=|{\cal T}^k|$, i.e., the cardinality of set ${\cal T}^k$.}
\STATE{Initialization:$u^0 \in \mathbb{R}^k$, $v^0 \in \mathbb{R}^m$, and $w^0 \in \mathrm{null}(A^T)$ (i.e., the null space of the transpose of  matrix $A$, say, $w^0 =0$).}
\STATE{ for $t=1,2, \ldots,$}
\STATE{\ \ update $u^t$ via \eqref{Eq:update-u-alg2},}
\STATE{\ \ update $v^{t}$ via \eqref{Eq:update-v-alg2},}
\STATE{\ \ $w^t = w^{t-1} + \gamma (v^t - Au^t)$.}
\STATE{End until the stopping criterion satisfied.}
\STATE{ {\bf Output}:
$u^t$.}
\end{algorithmic}
\end{algorithm}

\subsection*{Appendix B. Closed form solution of proximal of squared hinge}
Consider the following optimization problem
\begin{align}
\label{Eq:squaredhinge-min}
u^* = \arg\min_u g(u) := \left(\max\{0,1-a\cdot u\}\right)^2 + \frac{\gamma}{2}(u-b)^2,
\end{align}
where $\gamma>0$.
\begin{lemma}
\label{Lemm:solution-squaredhinge-min}
The optimal solution of the problem \eqref{Eq:squaredhinge-min} is shown as follows
\begin{align*}
\mathrm{hinge}_{\gamma}^2(a,b) =
&\left\{
\begin{array}{cl}
b, &\ \mathrm{if} \ a=0,\\
\frac{2a+\gamma b}{2a^2+\gamma}, & \mathrm{if}\ a\neq 0, ab<1\\
b, & \mathrm{if}\ a \neq 0, ab\geq 1.
\end{array}
\right.
\end{align*}
\end{lemma}

\begin{proof}
We consider the problem \eqref{Eq:squaredhinge-min} respectively in the following three scenarios: (1) $a>0$, (2) $a=0$ and (3) $a<0$.

\textbf{Case 1. $a>0$:}
In this case,
\begin{align*}
g(u)
= \left\{
\begin{array}{cl}
(1-au)^2+\frac{\gamma}{2}(u-b)^2, \ & u<a^{-1},\\
\frac{\gamma}{2}(u-b)^2, \ & \mathrm{u\geq a^{-1}}.
\end{array}
\right.
\end{align*}
It is easy to show that the solution of the problem is
\begin{align}
\label{Eq:hing-sol-part1}
u^* = \left\{
\begin{array}{cl}
\frac{2a+\gamma b}{2a^2+\gamma}, &\ \mathrm{if} \ a >0 \ \mathrm{and}\ b< a^{-1},\\
b, &\ \mathrm{if} \ a >0 \ \mathrm{and}\ b\geq a^{-1}.
\end{array}
\right.
\end{align}

\textbf{Case 2. $a=0$:}
It is obvious that
\begin{align}
\label{Eq:hing-sol-part2}
u^* =b.
\end{align}

\textbf{Case 3. $a<0$:}
Similar to \textbf{Case 1},
\begin{align*}
g(u)
= \left\{
\begin{array}{cl}
(1-au)^2+\frac{\gamma}{2}(u-b)^2, \ & u\geq a^{-1},\\
\frac{\gamma}{2}(u-b)^2, \ & {u < a^{-1}}.
\end{array}
\right.
\end{align*}
Similarly, it is easy to show that the solution of the problem is
\begin{align}
\label{Eq:hing-sol-part3}
u^* = \left\{
\begin{array}{cl}
\frac{2a+\gamma b}{2a^2+\gamma}, &\ \mathrm{if} \ a <0 \ \mathrm{and}\ b\geq a^{-1},\\
a^{-1}, &\ \mathrm{if} \ a <0 \ \mathrm{and}\ a^{-1}<b<a^{-1}-\gamma^{-1}a ,\\
b, &\ \mathrm{if} \ a <0 \ \mathrm{and}\ b < a^{-1}.\\
\end{array}
\right.
\end{align}
Thus, we finish the proof of this lemma.
\end{proof}

\ifCLASSOPTIONcaptionsoff
  \newpage
\fi

\end{document}